\documentclass{article}

\usepackage[final]{neurips_2023}

\usepackage[utf8]{inputenc} 
\usepackage[T1]{fontenc}    
\usepackage{url}            
\usepackage{booktabs}       
\usepackage{amsfonts}       
\usepackage{nicefrac}       
\usepackage{microtype}      
\usepackage{graphicx}
\usepackage[linesnumbered,ruled,vlined] {algorithm2e}
\usepackage{wrapfig}
\usepackage{mdframed}
\usepackage{amsmath}
\usepackage{amsthm}
\usepackage{booktabs}
\usepackage{amsfonts}
\usepackage{booktabs}

\newcommand{\lf}[1]{{\slshape \texttt{#1}}}

\input{pygment-settings}

\newtheorem{theorem}{Theorem}
\newtheorem{proposition}{Proposition}
\newtheorem{definition}{Definition}
\newtheorem{lemma}{Lemma}
\newenvironment{sproof}{
	\proof}{\endproof}

\newcommand{\mc}[1]{\mathcal{#1}}

\newcommand{\m}{M}
\newcommand{\agent}{\mathcal{A}}

\newcommand{\mysssection}[1]{\noindent\textbf{#1}\hspace{5pt}}

\newcommand{\psq}{q_\emph{PS}}

\newcommand{\psr}{\theta_\emph{PS}}
\newcommand{\tup}[1]{\langle {#1} \rangle}

\usepackage{bbm}
\usepackage[cmtip,all]{xy}
\newcommand{\longsquiggly}{\xymatrix{{}\ar@{~>}[r]&{}}}

\usepackage{siunitx}
\usepackage{makecell}
\usepackage{fancyvrb}
\usepackage{multirow}
\usepackage{booktabs}
\usepackage[table,dvipsnames]{xcolor} 
\usepackage[colorlinks=true,linkcolor=MidnightBlue,anchorcolor=black,citecolor=MidnightBlue,urlcolor=MidnightBlue, pdfborderstyle={/S/U/W 0.5}]{hyperref}

\usepackage{etoolbox}
\makeatletter
\patchcmd{\@algocf@start}
  {-1.5em}
  {0pt}
  {}{}
\makeatother

\newcommand{\name}{Query-based Autonomous Capability Estimation }
\newcommand{\nameAbbr}{QACE }

\input{pygment-settings}

\title{Autonomous Capability Assessment of Sequential Decision-Making Systems in Stochastic Settings\\
{\large(Extended Version)}}

\author{%
    Pulkit Verma, Rushang Karia, {\normalfont and} Siddharth Srivastava\\
    Autonomous Agents and Intelligent Robots Lab,\\
    School of Computing and Augmented Intelligence, \\
    Arizona State University, AZ, USA\\
    \{\texttt{verma.pulkit}, \texttt{rushang.karia}, \texttt{siddharths}\}\texttt{@asu.edu}
}

\begin{document}

\maketitle

\begin{abstract}
It is essential for users to understand what their AI systems can and can't do in order to use them safely. However, the problem of enabling users to assess AI systems with 
sequential decision-making (SDM) capabilities is relatively understudied. This paper presents a new approach for modeling the capabilities of black-box AI systems that can plan and act, along with the possible effects and requirements for executing those capabilities in stochastic settings. We present an active-learning approach that can effectively interact with a black-box SDM system and learn an interpretable probabilistic model describing its capabilities. Theoretical analysis of the approach identifies the conditions under which the learning process is guaranteed to converge to the correct model of the agent; empirical evaluations on different agents and simulated scenarios show that this approach is few-shot generalizable and can effectively describe the capabilities of arbitrary black-box SDM agents in a sample-efficient manner.
\end{abstract}

\section{Introduction}
\label{sec:introduction}

AI systems are becoming increasingly complex, and it is becoming difficult even for AI experts 
to ascertain the limits and capabilities of such systems, as they often use black-box policies for 
their decision-making process~\citep{popov2017data,greydanus_2018_visualizing}. 
E.g., consider an elderly couple with a household robot that learns and adapts to their specific household. 
How would they determine what it can do, what effects their commands would have, and under what conditions? 
Although we are making steady progress on learning for sequential decision-making (SDM), the problem of 
enabling users to understand the limits and capabilities of their SDM systems is largely unaddressed. Moreover, 
as the example above illustrates, the absence of reliable approaches for user-driven capability assessment of 
AI systems limits their inclusivity and real-world deployability. 

This paper presents a new approach for \emph{\name} (\nameAbbr\!) of black-box SDM systems in stochastic settings. Our approach uses a restricted form of interaction with the input SDM agent (referred to as SDMA) to learn a probabilistic model of its capabilities. The learned model captures  high-level user-interpretable capabilities, such as the conditions under which an autonomous vehicle could back out of a garage, or reach a certain target location, along with the probabilities of  possible outcomes of executing each such  capability.   
The resulting learned models directly provide interpretable representations of the scope of SDMA's capabilities. They can also be used to enable and support 
approaches for explaining SDMA's behavior that require closed-form models (e.g., ~\cite{Sreedharan2018HELM}). 
We assume that the input SDMA provides a minimal query-response interface that is already commonly supported by contemporary SDM systems. In particular, SDMA should reveal capability names defining how each of its capabilities can be invoked, and it should be able to accept user-defined instructions in the form of sequences of such capabilities.  These requirements are typically supported by SDM systems by definition. 

The main technical problem for \nameAbbr is to automatically compute ``queries'' in the form of instruction sequences and policies, 
and to learn a probabilistic model for each capability based on SDMA's ``responses'' in the form of executions. Depending on the scenario, these executions can be in the real world, or in a simulator for safety-critical settings. 
Since the set of possible queries of this form is exponential in the state space,  na\"{i}ve approaches for enumerating and selecting useful queries based on information gain metrics are infeasible. 

\mysssection{Main contributions} This paper presents the first approach for for query-based assessment of SDMAs in stochastic settings with minimal assumptions on SDMA internals. In addition, it is also the first approach for reducing query synthesis for SDMA assessment to full-observable non-deterministic (FOND) planning~\citep{Cimatti1998Strong}.
Empirical evaluation shows that these contributions enable our approach to carry out scalable assessment in both embodied and vanilla SDMAs.

We express the learned models using an input concept vocabulary that is known to the target user group. Such vocabularies span multiple tasks and environments. They can be acquired through parallel streams of research on interactive concept acquisition~~\citep{kim15_nips,Lage2020Learning} or explained to users through demonstrations and training~\citep{Schulze2000Andes}. 
These concepts can be modeled as 
binary-valued \emph{predicates} that have their associated evaluation functions~\citep{Mao2022PDSketch}.
We use the syntax and semantics of a well-established relational SDM model representation language, Probabilistic Planning Domain Definition Language (PPDDL)~\citep{Younes_2004_ppddl}, to express the learned models. 

Related work on the problem addresses model learning from passively collected observations of agent behavior~\citep{pasula2007learning,Martinez_2016_learning,juba_22_learning};
and by exploring the state space using simulators~\citep{chitnis2021glib,Mao2022PDSketch}.
However, passive learning approaches can learn incorrect models as they do not have the ability to
generate interventional or counterfactual data;
exploration techniques can be sample inefficient because they 
don't take into account uncertainty and incompleteness in the model being learned to guide their exploration
(see Sec.~\ref{sec:relatedwork} for a greater discussion).

In addition to the key contributions mentioned earlier, our results (Sec.~\ref{sec:evaluation}) show that the approaches for query synthesis in this paper do not place any additional requirements on black-box SDMAs but 
significantly improve the following factors:
(i) convergence rate and sample efficiency for learning relational models of SDMAs with complex capabilities, 
(ii) few-shot generalizability of learned models to larger environments, and 
(iii) accuracy of the learned model w.r.t. the ground truth SDMA capabilities.
convergence rate to the sound and complete model.

\section{Preliminaries}
\label{sec:background}

\mysssection{SDMA setup} We consider SDMAs that operate in stochastic and
fully observable environments. An SDMA can be represented as
a 3-tuple $\tup{\mc{X},\mc{C},\mc{T}}$, where $\mc{X}$ is the environment state
space that the SDMA operates in, $\mc{C}$ is the set of SDMA's capabilities 
(capability names, e.g., ``place object x at 
location y'' or ``arrange table x'')
that the SDMA can execute, and $\mc{T}: \mc{X} \times \mc{C} \rightarrow \mu \mc{X}$ 
is the stochastic black-box transition model determining the effects of
SDMA's capabilities on the environment. 
Here, $\mu \mc{X}$ is the space of 
probability distributions on $\mc{X}$.
Note that the semantics of $\mc{C}$ are not known to the user(s) 
and $\mc{X}$ may not be user-interpretable. 
The only information available about the SDMA is the instruction set in the form 
of capability  names, represented  as $\mc{C}_N$. 
This isn't a restricting assumption 
as the SDMAs must reveal their instruction sets for usability.

\mysssection{Running Example} 
Consider a cafe server robot 
that can 
pick 
and place items like plates, cans, 
etc., from various locations
in the cafe, like the counter, tables, etc., and also move between these 
locations.
A capability \lf{pick-item (?location ?item)} would allow a user to instruct the robot to pick up an item like a soda can
for any location. However, without knowing its description, the user would not know under what conditions the
robot could 
execute this capability and what the effects will be.

\mysssection{Object-centric concept representation} We aim to learn representations that are
generalizable, i.e., the transition dynamics learned should be impervious to environment-specific
properties such as numbers and configurations of objects. 
Additionally, the learned capability models
should hold in different settings of objects in the 
environment as long as the SDMA's capabilities does not change. To this
\begin{wrapfigure}{r}{0.34\textwidth}
    \begin{minipage}{0.34\textwidth}
    \centering
    \frame{\includegraphics*[width=\textwidth]{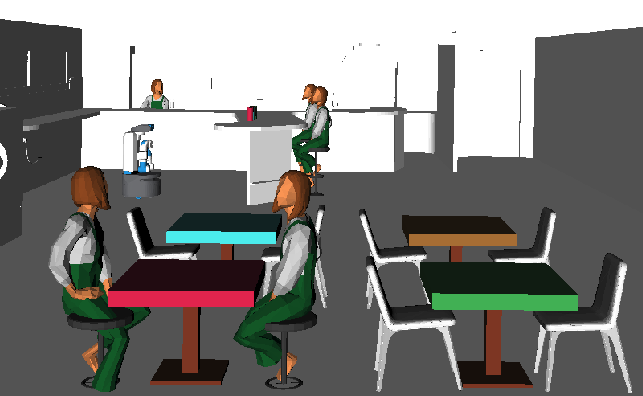}}
    \end{minipage}
    \caption{The cafe server robot environment in OpenRave simulator.}
    \label{fig:cafe_sim}
    \vspace{-0.35cm}
    \end{wrapfigure}
effect, we learn the SDMA's transition model in terms of interpretable
concepts that can be represented using first-order logic \emph{predicates}. This is a 
common formalism for expressing the symbolic models of SDMAs~\citep{zhi2020online,Mao2022PDSketch}.
 We formally represent them using a set of
object-centric predicates $\mc{P}$. 
The set of predicates used for cafe server robot in Fig.~\ref{fig:cafe_sim} can be \lf{(robot-at ?location)},
\lf{(empty-arm)}, \lf{(has-charge)}, \lf{(at ?location ?item)}, and \lf{(holding ?item)}.
Here, \lf{?} precedes an argument that can be replaced by an object in the environment. E.g., \lf{(robot-at tableRed)}
means ``robot is at the red table.'' As mentioned earlier, we assume these predicates along with their 
Boolean evaluation functions (which evaluate to true if predicate is true in a state)
are available as input. 
Learning such predicates is also an interesting but orthogonal direction of 
research~\citep{Mao2022PDSketch,sreedharan22_iclr,Das2023State2Explanation}.

\mysssection{Abstraction} Using an object-centric predicate representation
induces an abstraction of environment states $\mc{X}$ to high-level
logical states $\mc{S}$ expressible in predicate vocabulary $\mc{P}$.  
This abstraction can be formalized using a surjective 
function $f: \mc{X} \rightarrow \mc{S}$. 
E.g.,
in the cafe server robot, the concrete state $x$ may refer to
roll, pitch, and yaw values.
On the other hand, the abstract state $s$  corresponding to $x$ will consist of truth values
of all the predicates~\citep{srivastava2014combined,Srivastava_16_metaphysics,Mao2022PDSketch}.

\begin{wrapfigure}{r}{0.4\textwidth}
    \vspace{-0.55cm}
    \begin{minipage}{0.4\textwidth}
    \small
    \begin{mdframed}[backgroundcolor=white,innerleftmargin=-0.8cm,linecolor=white]
    \begin{Verbatim}[commandchars=\\\{\}]
    {(}\PY{k}{:capability}\PY{+w}{ }\PY{n+nx}{pick\PYZhy{}item}
    \PY{+w}{ }\PY{k}{:parameters}\PY{+w}{ }\PY{p}{(}\PY{n+nc}{?location}\PY{+w}{ }\PY{n+nc}{?item}\PY{p}{)}
    \PY{+w}{ }\PY{k}{:precondition}\PY{+w}{ }\PY{p}{(}\PY{n+nb}{and}
    \PY{+w}{   }\PY{p}{(}\PY{n+nx}{empty\PYZhy{}arm}\PY{p}{)}\PY{+w}{ }\PY{p}{(}\PY{n+nx}{has\PYZhy{}charge}\PY{p}{)}
    \PY{+w}{   }\PY{p}{(}\PY{n+nx}{robot\PYZhy{}at}\PY{+w}{ }\PY{n+nc}{?location}\PY{p}{)}
    \PY{+w}{   }\PY{p}{(}\PY{n+nx}{at}\PY{+w}{ }\PY{n+nc}{?location}\PY{+w}{ }\PY{n+nc}{?item}\PY{p}{)}\PY{p}{)}
    \PY{+w}{ }\PY{k}{:effect}\PY{+w}{ }\PY{p}{(}\PY{n+nb}{and}\PY{+w}{ }\PY{p}{(}\PY{n+nx}{probabilistic}
    \PY{+w}{   }\PY{l+m+mf}{0.7}\PY{+w}{ }\PY{p}{(}\PY{n+nb}{and}\PY{+w}{ }\PY{p}{(}\PY{n+nb}{not}\PY{+w}{ }\PY{p}{(}\PY{n+nx}{empty\PYZhy{}arm}\PY{p}{)}\PY{p}{)}
    \PY{+w}{        }\PY{p}{(}\PY{n+nb}{not}\PY{+w}{ }\PY{p}{(}\PY{n+nx}{at}\PY{+w}{ }\PY{n+nc}{?location}\PY{+w}{ }\PY{n+nc}{?item}\PY{p}{)}\PY{p}{)}\PY{+w}{ }
    \PY{+w}{        }\PY{p}{(}\PY{n+nx}{holding}\PY{+w}{ }\PY{n+nc}{?item}\PY{p}{)}\PY{p}{)}
    \PY{+w}{   }\PY{l+m+mf}{0.2}\PY{+w}{ }\PY{p}{(}\PY{n+nb}{and}\PY{+w}{ }\PY{p}{(}\PY{n+nb}{not}\PY{+w}{ }\PY{p}{(}\PY{n+nx}{has\PYZhy{}charge}\PY{p}{)}\PY{p}{)}\PY{p}{)}
    \PY{+w}{   }\PY{l+m+mf}{0.1}\PY{+w}{ }\PY{p}{(}\PY{n+nb}{and}\PY{p}{)}\PY{p}{)}\PY{p}{)}\PY{k}{          #No-change}
    \end{Verbatim}
    \vspace{-0.2cm}
    \end{mdframed}
    \end{minipage}
    \caption{PPDDL description for the cafe server robot's \emph{pick-item} capability.}
    \label{fig:ppddl}
    \vspace{-0.2cm}
    \end{wrapfigure}

\mysssection{Probabilistic transition model} 
Abstraction induces an abstract transition model 
$\mc{T}': \mc{S} \times \mc{C} \rightarrow \mu \mc{S}$, where $\mu \mc{S}$ is the space of 
probability distributions on $\mc{S}$. This is done by converting each transition 
$\tup{x,c,x'} \in \mc{T}$ to  $\tup{s,c,s'} \in \mc{T}'$ using predicate evaluators such that $f(x) = s$ and $f(x') = s'$.
Now, $\mc{T}'$ can be expressed as model $M$ that is a set of parameterized action (capability in our case) schema, where each
$c \in \mc{C}$ is described as $c = \tup{\emph{name(c)}, \emph{pre(c)}, \emph{eff(c)}}$,
where $\emph{name}(c) \in \mc{C}_N$ refers to name and arguments (parameters) of $c$; $\emph{pre(c)}$ refers to 
the preconditions of the capability $c$ represented as a conjunctive formula defined over $\mc{P}$ that must be true
in a state to execute $c$; and $\emph{eff(c)}$ refers to the set of conjunctive formulas over $\mc{P}$, each of which becomes true
on executing $c$ with an associated probability. 
The result of executing $c$ for a model $M$ is a state $c(s)= s'$
such that $P_{M}(s'|s,c)>0$ and one (and only one) of the effects of $c$ becomes true in $s'$.
We also use $\tup{s,c,s'}$ triplet to refer to $c(s)=s'$.
This representation
is similar to the Probabilistic Planning Domain Definition Language (PPDDL), which can compactly
describe the SDMA's capabilities. 
E.g., the cafe server robot 
has three capabilities (shown here as \lf{name(args)}): \lf{pick-item(?location ?item)};  \lf{place-item(?location ?item)}; and 
\lf{move(?source ?destination)}. The description of \lf{pick-item} in PPDDL is shown in 
Fig.~\ref{fig:ppddl}. 

\mysssection{Variational Distance} 
\label{sec:vd}
Given a black-box SDMA $\agent$, we learn the probabilistic model $M$
representing its capabilities. To measure how close $M$ is to the true SDMA transition model $\mc{T}'$,
we use variational distance -- a standard measure in probabilistic-model learning 
literature~\citep{pasula2007learning,Martinez_2016_learning,ng2019incremental,chitnis2021glib}.
It is based on the \textit{total variation 
distance} between two probability distributions $\mc{T}'$ and $M$,
given as:
\begin{equation}
\label{eq:vd}
\small
\delta(\mc{T}', M) =  \frac{1}{|\mathcal{D}|}\sum_{\langle s, c, s' \rangle \in \mathcal{D}} \big| P_{\mathcal{T}'}(s'|s, c) - P_{M}(s'|s, c)\big|
\end{equation}
where $\mathcal{D}$ is the set of test samples ($\tup{s,c,s'}$ triplets) that we generate using $\mc{T}'$ to measure the
accuracy of our approach.
As shown by \citet{Pinsker1964}, 
$\delta(\mc{T}', M) \leq \sqrt{0.5\times D_\emph{KL}(\mc{T}' \!\parallel\! M)}$, where $D_\emph{KL}$ is the KL divergence.

\section{The Capability Assessment Task}
\label{sec:problem}

In this work, we aim to learn a
probabilistic transition model $\mc{T}'$ of a
black-box SDMA as a model $M$,
given a set of user-interpretable concepts as predicates $\mc{P}$ along with their evaluation functions, and the
capability names $C_N$ corresponding to the SDMA's capabilities.
Formally, the assessment task is:
\begin{definition}
Given a set of predicates $\mc{P}$ along with their Boolean evaluation functions, capability names $\mc{C}_N$, and a 
black-box SDMA $\agent$ in a fully observable, stochastic, 
and static environment, the
\emph{capability assessment task} $\tup{\agent,\mc{P},\mc{C}_N,\mc{T}'}$ is defined as the task of
learning the probabilistic transition model $\mc{T}'$ of the SDMA $\agent$ expressed using $\mc{P}$.
\end{definition}

The solution to this task is a model $M$ that should ideally be the same as $\mc{T}'$
for correctness. In practice, $\mc{T}'$ need not be in PPDDL, so the correctness should be evaluated along multiple
dimensions.
 
\mysssection{Notions of model correctness}
\label{sec:correctness} 
As discussed in Sec.~\ref{sec:background}, variational distance is one way to capture the
correctness of the learned model. This is useful when the learned model and the SDMA's model
are not in the same representation.
The correctness of a model can also be measured using qualitative properties such as soundness and completeness.
The learned model $M$ should be sound and complete w.r.t. the SDMA's high-level model $\mc{T}'$, i.e., 
for all combinations of $c$, $s$, and $s'$, if a transition $\tup{s,c,s'}$ is possible according to $\mc{T}'$, then it
should also be possible under $M$, and vice versa. Here, $\tup{s,c,s'}$ is consistent with $M$ (or $\mc{T}'$) if $P(s'|s,c) >0 $ according to $M$ (or $\mc{T}'$).
We formally define this as:

\begin{definition}
\label{def:soundeness}
    Let $\tup{\agent,\mc{P},\mc{C}_N,\mc{T}}$ be a capability assessment task with a learned model $M$ as its solution. 
    $M$ is \emph{sound} iff each transition $\tup{s,c,s'}$ consistent with $M$ is also consistent with $\mc{T}'$.
    $M$ is \emph{complete} iff every transition that is consistent with $\mc{T}'$ is also consistent with $M$. 
\end{definition}

This also means that if $\mc{T}'$ is also a PPDDL model, then (i) any precondition 
or effect learned as part of $M$ is also present in $\mc{T}'$ (soundness), and; 
(ii) all the preconditions and effects present in $\mc{T}'$ should be present in $M$ (completeness).
Additionally, a probabilistic model is \emph{correct} if it is sound and complete, and the probabilities for each effect set in each
of its capabilities are the same as that of $\mc{T}'$. 

\section{Interactive Capability Assessment}
\label{sec:solution}
To solve the capability assessment task, we must identify 
the preconditions and effects of each capability 
in terms of conjunctive formulae expressed over $\mc{P}$. 
At a very high-level, we do this
by identifying that a probabilistic model can be expressed as a set of capabilities $c \in C$, each of which
has two places where we can add a predicate $p$, namely precondition
and effect. We call these \emph{locations} within each capability.
We then enumerate through these $2\times|\mc{C}|$ locations and
figure out the correct form of each predicate at each of those locations.
To do this we need to consider three forms: (i) adding it as $p$, i.e., the predicate
must be true for that capability to execute (when the location is precondition), or
it becomes true on executing it (when the location is effect); (ii) adding it as
$\emph{not(p)}$, i.e., the predicate
must be false for that capability to execute (when the location is precondition), or
it becomes false on executing it (when the location is effect); (iii) not adding it at all, i.e., the capability execution does not depend on it (when the location is precondition), or the capability
does not modify it (when the location is effect).

\mysssection{Model pruning}
Let $\mc{M}$ represent the set of all possible transition models expressible in terms of
$\mc{P}$ and $\mc{C}$. 
We must prune the set of possible models to solve the capability assessment task, ideally bringing it to a singleton. 
We achieve this by posing queries to the SDMA and using the responses to the queries as data to eliminate the inconsistent models from the set of possible models $\mc{M}$.

Given a location (precondition or effect in a capability), the set of models corresponding to a predicate will consist of 3 transition models: one each corresponding to the three ways we can add the predicate in that location. We call these three possible models
$M_T$, $M_F$, $M_I$, corresponding to adding $p$ (true), $\emph{not(p)}$ (false), and not adding $p$ (ignored), respectively at that location.

Note that the actual set of possible transition models is infinite due to the probabilities associated with each transition. To simplify this, we first 
constrain the set of possible models by ignoring the probabilities, and learn a non-deterministic transition model 
(commonly referred to as a FOND model~\citep{Cimatti1998Strong})
instead of a probabilistic one. We later learn the probabilities using
maximum likelihood estimation based on the transitions observed as part of the query responses.

\mysssection{Simulator use} Using the standard assumption of a simulator's availability in research on SDM, \nameAbbr solves the capability assessment task (Sec.~\ref{sec:problem}) by issuing queries to the SDMA and observing its responses in the form of its execution in the simulator. In non-safety-critical scenarios, this approach can
work without a simulator too.
The interface required to answer the queries is rudimentary as the SDMA $\agent$ need not have access to its transition model $\mc{T}'$ (or $\mc{T}$). Rather, it should be able to interact with the
environment (or a simulator) to answer the queries. 
We next present the types of queries we use, followed by algorithms for generating them and for inferring the SDMA's model using its responses to the queries.

\mysssection{Policy simulation queries ($Q_\emph{PS}$)} 
These queries ask the SDMA $\agent$ to execute a given policy multiple times. More precisely, a $Q_\emph{PS}$ query is a tuple $\tup{s_I, \pi, G, \alpha, \eta}$ where $s_I \in \mc{S}$ is a state, 
$\pi$ is a partial policy that maps each reachable state to a capability, $G$ is a logical predicate formula that expresses a stopping condition, $\alpha$ is an execution cutoff bound representing the maximum number of execution steps, and $\eta$ is an attempt limit. 
Note that the query (including the policy) is created entirely by our solution approach without any interaction with the SDMA.
$Q_\emph{PS}$ queries ask $\agent$ to execute $\pi$, $\eta$ times. In each iteration, execution continues until either the stopping goal condition $G$ or the execution bound $\alpha$ is reached. 
E.g., ``Given that the robot, \lf{soda-can}, \lf{plate1}, \lf{bowl3} are at \lf{table4}, 
what will happen if the robot follows the following policy: if there is an item on the table and arm is empty, pick up the item; if an item is in the hand and location is not dishwasher, move to the dishwasher; if an item is in the hand and location is dishwasher, place the item in the dishwasher?'' Such queries will be used
to learn both preconditions and 
effects (Sec.~\ref{sec:learningmodels}). 

A response to such queries 
is an execution in the simulator and $\eta$ traces of these simulator executions. 
Formally, the response $\psr$ for a query $\psq \in Q_\emph{PS}$ is a tuple $\langle b, \zeta \rangle$, where $b \in \{\top,\bot\}$
indicates weather if the SDMA reached a goal state $s_G \models G$, and 
$\zeta$ are the corresponding triplets $\tup{s,c,s'}$ generated during the $\eta$ policy executions. If the SDMA reaches $s_G$ even once during the $\eta$ simulations, $b$ is $\top$, representing that the goal can be reached using this policy. 
Next, we discuss how these responses are used to prune the set of possible models and learn the correct transition model of the SDMA.

\subsection{\name (QACE) Algorithm}

  \begin{wrapfigure}{r}{0.535\textwidth}
  \vspace{-0.5cm}
    \begin{minipage}{0.535\textwidth}
        \IncMargin{0.8em}
\begin{algorithm}[H]
    \DontPrintSemicolon
    \SetKwInOut{Input}{Input}
    \SetKwInOut{Output}{Output}
    \caption{\nameAbbr Algorithm}
    \label{alg:qace}
    \Input{\,\,\,predicates $\mc{P}$; capability names $\mc{C}_N$; \\\,\,\,state $s$; SDMA $\agent$; hyperparameters $\alpha, \eta$;\\\,\,\, FOND Planner $\rho$}
    \Output{\,\,\,$M$}
   
    $L \leftarrow \{\emph{pre}, \emph{eff}\} \times \mc{C}_N$\;
    $ M^* \gets$ initializeModel~($\mc{P},\mc{C}_N$)\;
    \For{\emph{each} $\langle l,p\rangle \in\langle L, \mc{P}\rangle$}{
        Generate $M_T,M_F,M_I$ by setting $p$ at $l$ in $M^*$ \;
        \For{\emph{each pair} $M_i, M_j$ \emph{in} $\{M_T,M_F,M_I\}$}{
            $q \leftarrow$ generateQuery$(M_i, M_j, \alpha, \eta, s, \rho$)\;
            $\theta_{\agent}, \mathbb{S} \leftarrow$ getResponse$(q, \agent, s)\,\,\,$ \;
            $M^* \leftarrow$ pruneModels ($\theta_\agent, M_i, M_j$) \;
            $M^* \leftarrow$   learn possible stochastic effects of capability with $c_N$ in $l$ using $\zeta$ (in $\theta_\agent$)\;
            }   
    }
    $M \gets$ learnProbabilitiesOfStochasticEffects($\zeta, M^*$) \;
    \vspace*{-0.15in}
    \Return $M$\;
\end{algorithm}
\end{minipage}
  \vspace{-0.6cm}
\end{wrapfigure}

We now discuss how we solve the capability assessment task using the \name algorithm (Alg.~\ref{alg:qace}), which works in two phases.
In the first phase, 
\nameAbbr learns all capabilities' preconditions and non-deterministic effects using the policy simulation queries (Sec.~\ref{sec:query_synthesis}). 
In the second phase, 
\nameAbbr converts the
non-deterministic effects of capabilities into probabilistic effects (Sec.~\ref{sec:learningmodels}). We now explain the
learning portion (lines 3-11) in detail.

\nameAbbr first initializes a model $M^*$ over the predicates in $\mc{P}$ with capabilities having names $c_N \in \mc{C}_N$.
All the preconditions and effects for all capabilities are empty in this initial model.
\nameAbbr uses $M^*$ to maintain the current partially learned model.
\nameAbbr iterates over all combinations of
$L$ and $\mc{P}$ (line 4). 
For each pair, \nameAbbr creates 3 candidate models $M_T$, $M_F$, and $M_I$ as mentioned earlier.
It then takes 2 of these (line 5) and generates a query $q$ (line 6) such that 
responses to the query $q$ from the two models are logically inconsistent
 (see Sec.~\ref{sec:query_synthesis}).
The query $q$ is then posed to the SDMA $\agent$ whose response is stored as $\theta_\agent$ (line 7). \nameAbbr finally prunes at least one of the two models by comparing their responses (which are logically inconsistent) with the response $\theta_\agent$ of the SDMA
on that query (line 8). \nameAbbr also updates the effects of all models in the set of possible models to speed up the learning process (line 9). Finally, it learns the probabilities of the observed stochastic effects using maximum likelihood estimation (line 10).
An important feature of the algorithm (similar to PLEX~\citep{Mehta2011Autonomous} and AIA~\citep{verma2021asking}) is that it keeps track of all the locations where it
hasn't identified the correct way of adding a predicate. The next section presents our approach for generating the queries in line 6.

\subsection{Algorithms for Query Synthesis}
\label{sec:query_synthesis}
One of the main challenges in interactive model learning is to generate the queries we discussed above and to learn the agent's model using them. 
Although active learning~\citep{settles2012ActiveLearning} addresses the related problem of figuring out which data sets to request labels for, vanilla active learning 
approaches are not directly applicable here because
the possible set of queries expressible using the literals in a domain is vast: it is the set of all policies expressible using $\mc{C}_N$. 
Query-based learning approaches use an estimate of the utility of a query to select ``good'' queries. This can be a multi-valued measure
like \emph{information gain}~\citep{Sollich1994Learning}, \emph{value}~\citep{macke2021expected}, etc.
or a binary-valued attribute like \emph{distinguishability}~\citep{verma2021asking}, etc.
We use distinguishability
as a measure to identify useful queries. According to it,
a query $q$ is distinguishing w.r.t. two models if
 responses by both models to $q$ are logically inconsistent.
We now discuss methods for generating such queries.

\mysssection{Generating distinguishing queries}
\nameAbbr automates the generation of queries using search. As part of the
algorithm, a model $M^*$ is used to generate the three possible 
models $M_T, M_F,$ and $M_I$
for a specific predicate $p$ and location $l$ combination.
Other than the predicate $p$ at location $l$, these models are exactly the same.
A forward search is used to generate the policy simulation queries with two possible models
$M_i, M_j$ chosen randomly from $M_T$, $M_F$, and $M_I$.
The forward search is initiated with 
an initial state $\tup{s_0^i,s_0^j}$ as the root of the search tree, where $s_0^i$ and $s_0^j$
are copies of the same state $s_0$ from which we are starting the search. The edges of the tree correspond to
the capabilities with arguments replaced with objects in the environment. Nodes
correspond to the two states resulting from applying the capability in the parent state
according to the two possible models.
E.g., consider that a transition $\tup{s_0^i,c,s_1^i}$ is possible according to the model $M_i$, and let 
$\tup{s_0^j,c,s_1^j}$ be the corresponding transition (by applying the same effect set of $c$ as $h_i$) according to the model $M_j$. 
Now there will be an edge in the forward search tree
with label $c$ such that parent node is $\tup{s_0^i,s_0^j}$ and child node is $\tup{s_1^i,s_1^j}$. 
The search process terminates when a node $\tup{s^i,s^j}$ is reached such that either the states $s^i$ and $s^j$
don't match, or the preconditions of the same capability were met in the state according to one of the possible models but not
according to the other. 
Vanilla forward search scales poorly with the number of capabilities and objects in the environment.
To address this we reduce the problem to a fully observable non deterministic (FOND) planning problem. This can be solved by any FOND planner.
The output of this search is a policy $\pi$
to reach a state where 
the two models, $M_i$ and $M_j$ predict different outcomes.
Additional details about the reduction and an example of the output policy are available in 
Appendix~\ref{app:psq_example}. 
The query $\tup{s_I, \pi, G, \alpha, \eta}$ resulting from this search is such that
$s_I$ is set to the initial state $s_0$, $\pi$ is the output policy, $G$ is the goal state where the models' responses doesn't match, $\alpha$ and $\eta$ are hyperparameters as mentioned earlier. We next see how to use these queries
to prune out the incorrect models.

\subsection{Learning Probabilistic Models Using Query Responses}
\label{sec:learningmodels}
At this point, \nameAbbr already has a query such that the response to the query by the two possible models
does not match. We next see how to prune out the model inconsistent with the SDMA.
\nameAbbr poses the query generated earlier to
the SDMA and gets its response (line 7 in Alg.~\ref{alg:qace}). If the SDMA can successfully execute the policy, \nameAbbr compares the response of the two models with that of the SDMA and prunes out the model
whose response does not match with that of the SDMA. If the SDMA cannot execute the policy, i.e., 
SDMA fails to execute some capability in the policy, then the models cannot be pruned directly.
In such a case, a new initial state $s_0$ must be chosen to generate a new query starting from that initial state.
Since generating new queries for the same pair of models can be time consuming, we preempt this
 issue by creating a pool of states  $\mathbb{S}$
that can execute the capabilities using a directed exploration of the state space with the current partially learned model as discussed below.

\mysssection{Directed exploration}
 A partially learned model is a model where one or more capabilities have been learned (the correct preconditions have been identified for each capability and at least one effect is learned). Once we have such a model, we can do a directed exploration of the state space for these capabilities by only executing a learned capability if the preconditions are satisfied. This helps in reducing the sample complexity since the simulator is only called when we know that the capability will execute successfully, thereby allowing us to explore different parts of the state space efficiently. If a capability's preconditions are not learned, all of its groundings might need to be executed from the state.
In the worst case, to escape local minima where no models can be pruned, we would need to perform a randomized search for a state where a capability is executable by the SDMA. In practice, we observed that using directed exploration to generate a pool of states gives at least one grounded capability instance. This helps ensure that during query generation, the approach does not spend a long time searching for a state where a capability is executable.

\mysssection{Learning probabilities of stochastic effects}
After \nameAbbr learns the non-deterministic model, to learn the probabilities of
the learned effects it uses the transitions
collected as part of responses to queries. This is done using Maximum Likelihood Estimation 
(MLE)~\citep{Fisher1922MLE}. For each triplet $\tup{s,c,s'}$ seen in the 
collected data, let $\emph{count}_c$ be the number of times a capability
$c$ is observed. Now, for each effect set, the probability of that effect set
becoming true on executing that capability $c$ is given as the number of times
that effect is observed on executing $c$ divided by $\emph{count}_c$.
As we increase the value of the hyperparameter $\eta$, we increase the number 
of collected triplets, thereby improving the probability values calculated using this approach.

\section{Theoretical Analysis and Correctness}
\label{sec:theory_results}

We now discuss how the model $M$ of SDMA $\agent$ learned using \nameAbbr 
fulfills the notions of correctness (Sec.~\ref{sec:correctness}) discussed earlier.
We first show that the model $M^*$ learned before line 10 of \nameAbbr (Alg.~\ref{alg:qace}) is sound
and complete according to Def.~\ref{def:soundeness}. The proofs for the theorems are available in 
Appendix~\ref{appendix:proofs}.

\begin{theorem}
Let $\agent$ be a black-box SDMA with a ground truth transition model $\mc{T}'$ expressible in terms of predicates $\mc{P}$ and a set of capabilities $\mc{C}$. Let $M^*$ be the non-deterministic model expressed in terms of predicates $\mc{P}^*$ and capabilities $\mc{C}$, and learned using the query-based autonomous capability estimation algorithm (Alg.~\ref{alg:qace}) just before line 10. Let $C_N$ be a set of capability names corresponding to capabilities $\mc{C}$. If $\mc{P}^* \subseteq \mc{P}$, then the model $M^*$ is \emph{sound} w.r.t. the SDMA transition model $\mc{T}'$. Additionally, if $\mc{P}^* = \mc{P}$, then the model $M^*$ is \emph{complete} w.r.t. the SDMA transition model $\mc{T}'$.
\end{theorem}

Next, we show that the final step of learning the probabilities for all the effects in each capability
converges to the correct probability distribution under the assumption that all the effects
of a capability are identifiable. When a capability $c$ is executed in the environment, one of its effects $e_i(c) \in \emph{eff}(c)$ will be observed in the environment.
To learn the correct probability distribution in $M$, we should accurately identify that effect $e_i(c)$.
Hence, the set of effects is \emph{identifiable} if at least one state exists in the
environment from which each effect can be uniquely identified when
the capability is executed.
An example of this is available in 
Appendix~\ref{appendix:ident_effects}.

\begin{theorem}
  Let $\agent$ be a black-box SDMA with a ground truth transition model $\mc{T}'$ expressible in terms of predicates $\mc{P}$ and a set of capabilities $\mc{C}$. Let $M$ be the probabilistic model expressed in terms of predicates $\mc{P}^*$ and capabilities $\mc{C}$, and learned using the query-based autonomous capability estimation algorithm (Alg.~\ref{alg:qace}). Let $\mc{P} = \mc{P}^*$ and $M$ be generated using a sound and complete non-deterministic model $M^*$ in line 11 of Alg.~\ref{alg:qace}, and let all effects of each capability $c \in \mc{C}$ be identifiable. The model $M$ is \emph{correct} w.r.t. the model  $\mc{T}'$ in the limit as $\eta$ tends to $\infty$, where $\eta$ is hyperparameter in query $Q_\emph{PS}$ used in Alg.~\ref{alg:qace}.
\end{theorem}

\section{Empirical Evaluation}
\label{sec:evaluation}

We implemented Alg.~\ref{alg:qace} in Python to evaluate our approach empirically.\footnote{Source code available at \url{https://github.com/AAIR-lab/QACE}}
We found that our query synthesis and interactive learning process leads to 
(i) few shot generalization; 
(ii) convergence to a sound and complete model; and
(iii) much greater 
sample efficiency and 
accuracy for learning lifted SDM models with complex capabilities as compared to the baseline.

\begin{figure*}[t]
\centering
    \includegraphics[width=\textwidth]{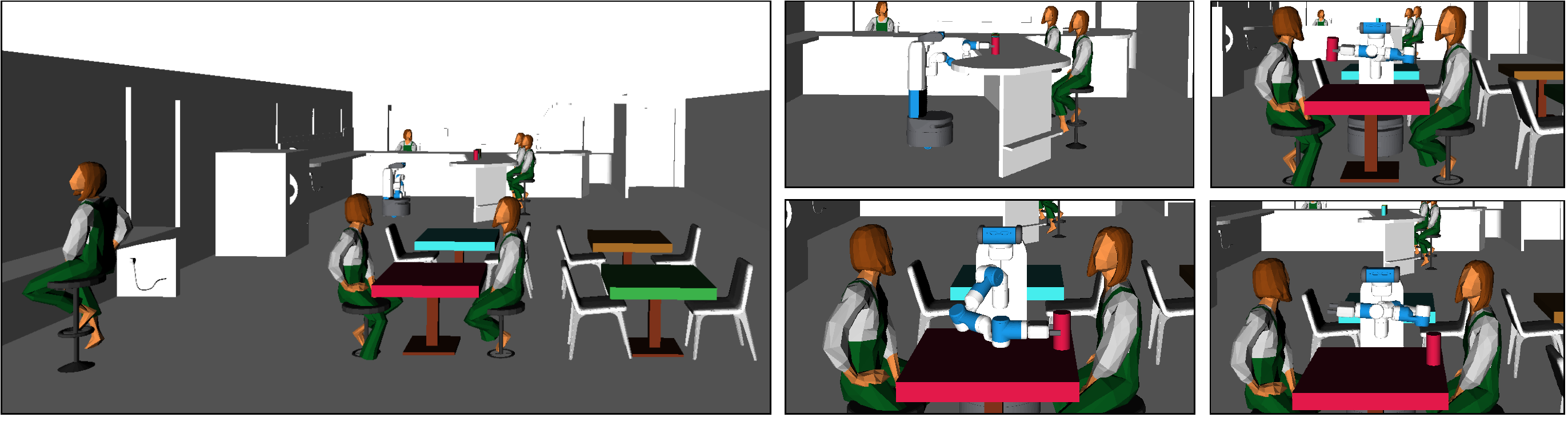}
    \caption{Screen captures from the Cafe Server Robot simulation. The complete environment is shown in the image on the left. The image grid on the right shows screen captures of multiple steps of the robot delivering a \lf{soda-can} to a table.}
    \label{fig:cafe_snaps}
\end{figure*}

\mysssection{Setup} We used a \emph{single} training problem with few objects ($\le7$) for all methods in our evaluation and used a test set that was composed of problems containing object counts larger than those in the training set.
We ran the experiments
on a cluster of Intel Xeon E5-2680 v4 CPUs with CentOS 7.9 running at 2.4 GHz with a memory limit of 8 GB and a time limit of 4 hours. 
We used PRP~\citep{Muise2012PRP} as the FOND planner to generate the queries (line 6 in Alg.~\ref{alg:qace}).
For QACE, we used $\alpha=2d$ where $d$ is the maximum depth of policies used in queries generated by \nameAbbr and $\eta=5$. All of the methods in our empirical evaluation receive the same training and test sets and are evaluated on the same platform. 
We used Variational Distance (VD) as presented in Eq.~\ref{eq:vd} to evaluate the quality of the learned SDMA models.

\mysssection{Baseline selection}
We used the closest SOTA related work, GLIB~\citep{chitnis2021glib} as a baseline. GLIB learns a probabilistic model of an intrinsically motivated agent by sampling
goals far away from the initial state and making the agent try to reach them. 
This can be adapted to an assessment setting by moving goal-generation based sampling outside the agent, and, to the best of our knowledge, no existing approach addresses the problem of creating intelligent questions for an SDMA.
GLIB has two versions, GLIB-G, which learns the model as a set of grounded noisy deictic rules (NDRs)~\citep{pasula2007learning}, and GLIB-L, which learns the model as a set of lifted NDRs. We used the same hyperparameters as published for the \textit{Warehouse Robot} and \textit{Driving Agent} and performed extensive tuning for the others and report results with the best performing settings.

The models learned using GLIB cannot be used to calculate the variational distance presented in Eq.~\ref{eq:vd} because for each capability GLIB learns a set of NDRs rather than a unique NDR. In order to maintain parity in comparison, we use GLIB's setup to calculate an approximation of the VD. Using it, we sample 3500 random transitions $\langle s, c, s'\rangle$ from the ground truth transition model $\mc{T}'$ using problems in the test set to compute a dataset of transitions $\mathcal{D}$. 
The sample-based, approximate VD is then given as:
$\frac{1}{|\mc{D}| } \sum_{d \in \mc{D}} \mathbbm{1}_{[s' \neq c_M(s)]}$, where $c_M(s)$ samples the transition using the capability in the learned model output by each method.
In Fig.~\ref*{fig:plots}, we compare the approximate variational distance of the three approaches w.r.t. $\mathcal{D}$ as we increase the learning time. 
Note that we also evaluated VD for QACE using Eq.~\ref{eq:vd} and found that $\delta(\mc{T}', M) \approx 0$ for our learned model $M$ in all SDMA settings.

\mysssection{SDMAs for evaluation} To test the efficacy of our approach, we created SDMAs for five different settings including one task and motion planning agent and several SDMAs based on state-of-the-art stochastic planning systems from the literature: 
\emph{Cafe Server Robot} is a Fetch robot~\citep{wise16_fetch} that uses the ATM-MDP task and motion planning system~\citep{shah2020icra}  to  plan and act in
a restaurant environment to serve food, clear tables, etc.; \emph{Warehouse Robot} is a robot that can stack, unstack, 
and manage the boxes in a warehouse; a \emph{Driving Agent} that can drive between locations and can repair the vehicle at certain locations;
a \emph{First Responder Robot} that can assist in emergency scenarios by driving to emergency spots, providing first-aid and water to victims, etc.; and an \emph{Elevator Control Agent} that can control the operation of multiple elevators in a building.

Additional details about each setting and are available in 
Appendix~\ref{appendix:domains}.

\subsection{Results}
We present an analysis of our approach on all of the SDMAs listed above. We also present a comparative analysis with the baseline on all  SDMAs except the Cafe Server Robot, whose task and motion planning system was not compatible with the baseline. 

\mysssection{Cafe Server Robot} This SDMA setup uses an 8 degrees of freedom Fetch~\citep{wise16_fetch} robot in a cafe setting on OpenRave simulator~\citep{Diankov_2008_openrave}. The low-level environment state consists of 
continuous x, y, z, roll, pitch, and yaw values of all objects in the environment. The predicate evaluators were
provided by ATM-MDP of which we used only a subset to learn a PPDDL model. Each robot capability is refined into
motion controls at run-time depending on the configuration of the objects in the environment. The results for
variational distance between the learned model and the ground truth model in Fig.~\ref{fig:cafe_sim_result} show that despite the different vocabulary,
\nameAbbr learns an accurate transition model for the SDMA. 

We now discuss the comparative performance of QACE with the baseline across the four baseline-compatible SDMAs presented above. 

\begin{wrapfigure}{r}{0.35\textwidth}
    \vspace{-0.1cm}
    \begin{minipage}{0.35\textwidth}
    \includegraphics*[width=\textwidth]{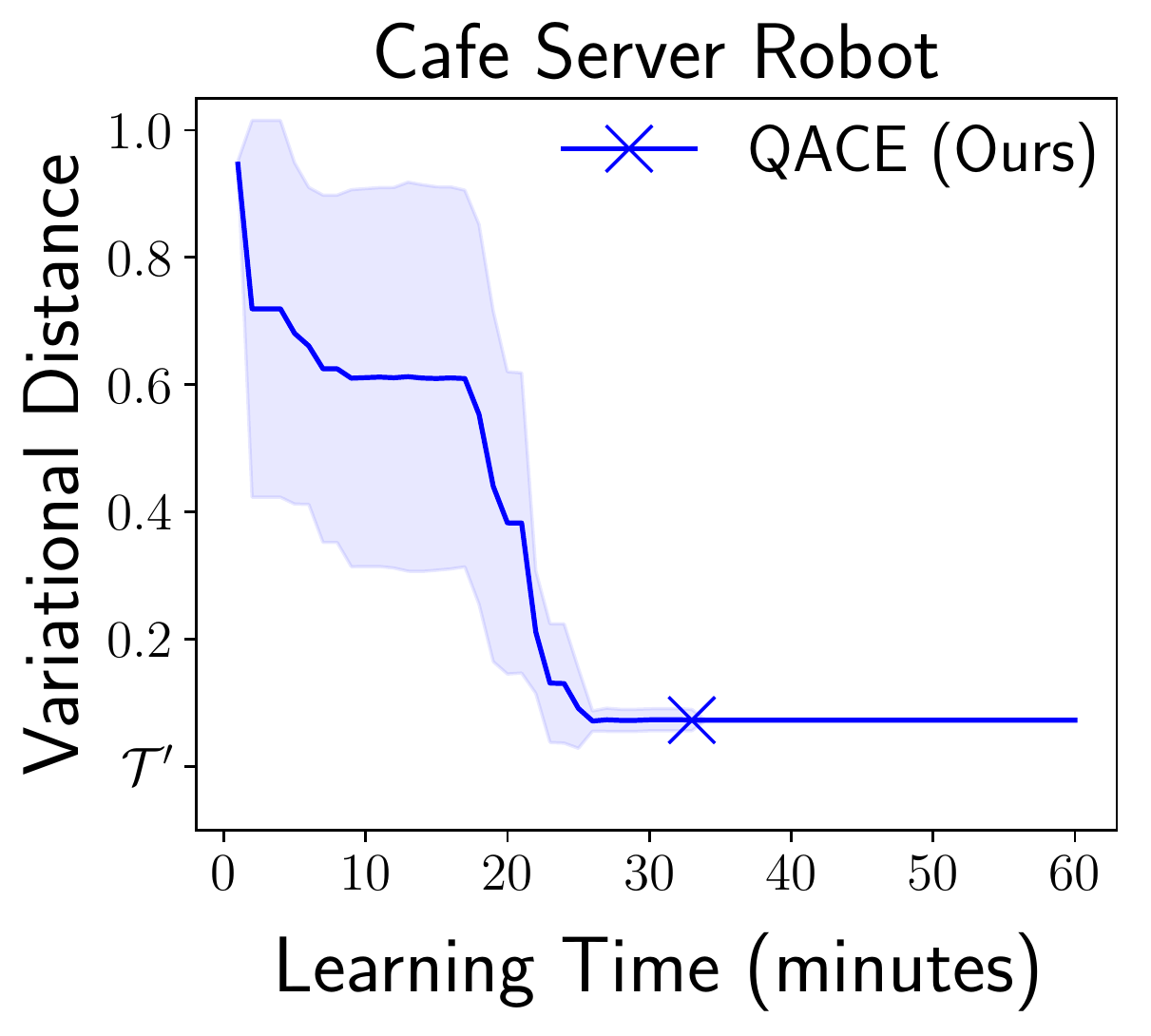}
    \end{minipage}
    \caption{Variational Distance between the learned model and the ground truth with increasing time for QACE for Cafe Server Robot. {\color{blue} $\mathbf{\times}$} shows that the learning process ended
    at that time instance.}
    \label{fig:cafe_sim_result}
    \vspace{-0.25cm}
\end{wrapfigure}

\mysssection{Faster convergence} The time taken for \nameAbbr to learn the final model is much lower than 
that of GLIB for three of the four SDMAs. This is because trace collection by \nameAbbr is more directed and hence ends up
learning the correct model in \emph{a shorter time}. The only setup where GLIB marginally outperforms \nameAbbr is Warehouse Robot, and this happens because this SDMA has just two capabilities, one of which is deterministic. Hence, GLIB can easily learn their configuration from a few observed traces. For SDMAs with complex and much larger number of capabilities -- First Responder Robot and Elevator Control Agent -- GLIB finds it more challenging to learn the model that is closer to the ground truth transition model. Additionally, \nameAbbr takes much fewer samples to learn the model than the baselines. In all settings, \nameAbbr is much more \emph{sample efficient} than the baselines as \nameAbbr needed at most 4\% of the samples needed by GLIB-G to reach the variational distance that it plateaued at. In contrast, GLIB-L started timing out only after processing a few samples for complex SDMAs.

\mysssection{Few-shot generalization} To ensure that learned models are not overfitted, our test set contains problems with larger quantities of objects than those used during training.  As seen in Fig.~\ref{fig:plots}, 
the baselines have higher variational distance from the ground truth model for complex SDMA setups as compared to \nameAbbr.
This shows \emph{better few-shot generalization} of \nameAbbr compared to the baselines.

\begin{figure*}[t]
\centering
    \includegraphics[width=\textwidth]{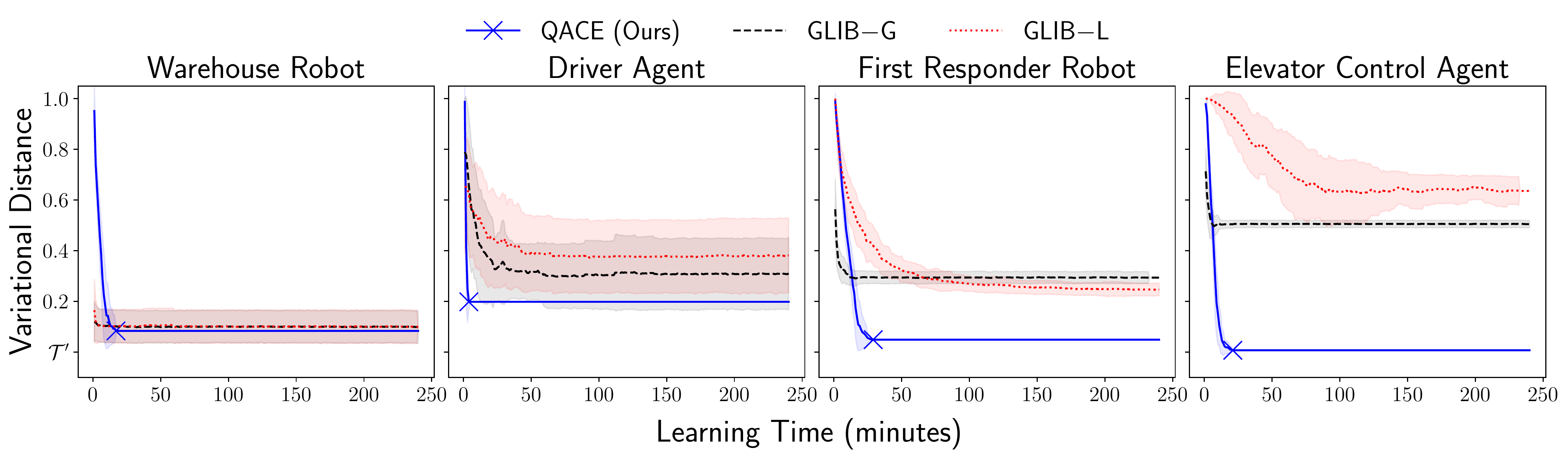}
    \caption{ 
    A comparison of the approximate variational distance as a factor of the learning time for the three methods: \nameAbbr (ours), GLIB-G, and GLIB-L (lower values better). {\color{blue} $\mathbf{\times}$} shows that the learning process ended
    at that time instance for \nameAbbr\!.
    The results were calculated using 30 runs per method per domain. Solid lines are averages across runs, and shaded portions show the standard deviation. $\mathcal{T}'$ is the ground truth model. Detailed results are available in Appendix~\ref{appendix:results}.}
    \label{fig:plots}
\end{figure*}

\section{Related Work}
\label{sec:relatedwork}

The problem of learning probabilistic relational agent models from a given set of observations has been well studied~\citep{pasula2007learning,mourao_2012_learning,Martinez_2016_learning,juba_22_learning}. 
\citet{jiminez_2012_review} and \citet{arora_2018_review} present comprehensive reviews of such approaches. 
We next discuss the closest related research directions.

\mysssection{Passive learning} Several methods learn a probabilistic
model of the agent and environment from a given set of agent executions. 
\citet{pasula2007learning}    
learn the models in the form of noisy deictic rules (NDRs) where an action
can correspond to multiple NDRs and also model noise.
\citet{mourao_2012_learning}
learn such operators using action classifiers to predict the effects
of an action.
\citet{Rodrigues_2011_incremental} learn non-deterministic models
as a collection of rule sets and learn these rule sets incrementally.
They take a bound on the number of rules as input.
\citet{juba_22_learning} provide a theoretical framework to learn
safe probabilistic models with a range of probabilities for each
probabilistic effect while assuming that each effect is atomic and
independent of others. A common issue with such approaches is that they
are susceptible to incorrect and sometimes inefficient model learning
as they cannot control the input data used for learning or carry out interventions required for accurate learning.

\mysssection{Sampling of transitions} 
Several approaches learn operator descriptions by exploring the state space in the restricted setting of deterministic models~\citep{ng2019incremental,Jin2022Creativity}.
A few reinforcement learning approaches have 
been developed for learning the relational
probabilistic action model by exploring the state space
using pre-determined criteria  
to generate better samples~\citep{ng2019incremental}. 
\citet{konidaris_18_from} explore learning
PPDDL models for planning, but they aim to learn
the high-level symbols needed to describe a set 
of input low-level options, and these symbols are not interpretable.
GLIB~\citep{chitnis2021glib} also learns probabilistic relational models using goal sampling
as a heuristic for generating relevant data, whereas we use active querying using guided forward search
for this. Our empirical analysis shows that our approach of synthesising  
queries yield greater sample efficiency and correctness profiles than 
the goal generation used in this approach.

\mysssection{Active learning} Several active learning 
approaches
learn automata representing a system's model~\citep{Angulin1988Queries,Aarts2012Automata,Tang2013Active,vaandrager_2017_model}. These approaches assume access to a teacher (or an oracle)
that can determine whether the learned automaton is correct and provide a counterexample if it is incorrect. This is not possible 
in the black-box SDMA settings that constitute the focus of this work.

\section{Conclusion}
\label{sec:conclusion}

In this work, we presented an approach for learning
a probabilistic model of an agent using interactive querying.
We showed that the approach is few-shot generalizable to larger environments and learns a sound and complete model faster than state-of-the-art approaches
in a sample-efficient manner.

QACE describes the capabilities of the robot in terms of predicates that the user understands (this includes novice users as well as more advanced users like engineers). Understanding the limits of the capabilities of the robot can help with the safe usage of the robot, and allow better utilization of the capabilities of the robot. Indirectly, this can reduce costs since the robot manufacturer need not consider all possible environments that the robot may possibly operate in. The use of our system can also be extended to formal verification of SDMAs.

QACE can also be used by standard explanation generators as they need an agent’s model. Such models are hard to obtain (as we also illustrate in this paper) and our approach can be used to compile them when they are not available to start with.

\mysssection{Limitations and Future Work}
In this work, we assume that the agent can be connected to a simulator to answer the queries. In some real-world settings, 
this assumption may be limiting as users might not have direct access to such a simulator. 
Formalizing the conditions under which is it safe to ask the
queries directly to the agent in the real-world is a promising direction for the future work.
Additionally, in this work, we assume the availability of the instruction set of the SDMA as input in the form of capability names.
In certain settings, it might be useful to discover the capabilities of an evolving SDMA using methods proposed by \citet{nayyar2022differential} and \citet{verma2022discovering}.

\acksection

We thank Jayesh Nagpal for his help with setting up the Cafe Server Robot SDMA. 
We also thank anonymous reviewers for their valuable feedback and suggestions.
This work was supported by the ONR under grants N00014-21-1-2045 and 
N00014-23-1-2416.

{
\bibliographystyle{plainnat}
\bibliography{assessment}
}

\clearpage
\appendix

\section{SDM Setups -- Additional Information}
\label{appendix:domains}

\begin{wraptable}{r}{0.41\textwidth}
    \vspace{-1.35cm}
    \begin{minipage}{0.41\textwidth}
    \rowcolors{2}{gray!13}{}
     \begin{tabular}{l *{2}{c} }    
    \toprule
    \textbf{SDM Setup} & $\mathbf{|\mc{P}|}$ & $\mathbf{|\mc{C}_N|}$  \\
    \midrule
       Cafe Server Robot & 5 & 4  \\
       Warehouse Robot & 8 & 4  \\
       Driver Agent & 4 & 2 \\
       First Responder Robot & 13 & 10 \\
       Elevator Control Agent & 12 & 10 \\
    \bottomrule
    \end{tabular}
    \caption{Size of the SDM setups in terms of number of predicates and capabilities.}
    \label{tab:size}
\end{minipage}
\vspace{-0.45cm}
\end{wraptable}

We used five SDM setups for our experiments, and Tab.~\ref{tab:size} shows their size in terms of number of predicates and capabilities. Description for the cafe server robot is available in Sec.~\ref{sec:evaluation}.
Short descriptions of the other four SDM settings are presented below:

\mysssection{Warehouse Robot}
This SDM setup is implemented using the SOTA stochastic planning system used in the planning literature. This is motivated from 
\textit{Exploding Blocksworld} setup introduced in the probabilistic track of International Planning Competition (IPC) 2004~\citep{younes_2005_first_IPPC}.
It features 
a robot that has four capabilities: \lf{stack}, \lf{unstack}, \lf{pick}, and \lf{place}.
\textit{stack} capability stacks one object on top of another, \lf{unstack} capability 
removes an object from top of another object, \lf{pick} capability picks up 
an object from a fixed location, and \lf{place} capability places the object
at a fixed location. The setup is non-deterministic as executing some of these capabilities can destroy the object as they might be delicate. Hence even the ground truth does not have 100\% success rate in this setup.

\mysssection{Driver Agent}
This SDM setup is implemented using the SOTA stochastic planning system used in the planning literature. This is motivated from 
\textit{Tireworld} setup introduced in the probabilistic track of IPC 2004~\citep{younes_2005_first_IPPC}.
It consists of a robot moving around multiple locations. The \lf{move-vehicle} capability that takes the SDMA from one location to another can also cause
it to get a flat-tire with some non-zero probability. Not all locations have the option to change tire, but if available,
a \lf{change-tire} capability will fix the flat-tire with a 100\% probability.

\mysssection{First Responder Robot}
This SDM setup is inspired from \textit{First Responders} in uncertainty track of IPC 2008~\citep{Bryce20086th}.
The setup features two kinds of emergencies: fire and medical, involving hurt victims. Victims can be treated at the site of an emergency or the hospital.
This was originally a FOND setup, and we added probabilities to all the capabilities with non-deterministic effects to make it probabilistic. The responder vehicles can also be driven from one place to another and can be loaded and unloaded with fire or medical kits.
The recovery status depending on the treatment location, is different with different probabilities. 

\mysssection{Elevator Control Agent}
This SDM setup is motivated from \textit{Elevators} in the probabilistic track of IPC 2006~\citep{Bonet20055th}.
It consists of an agent managing multiple elevators on multiple floors in a single building. The capabilities of moving from one elevator to another on the same floor are
probabilistic. The size of this setup is much larger than the previous three. Also, the capabilities have arities of up to five, making this setup complex
from an assessment point of view.

\section{Extended Preliminaries}

\subsection{Fully Observable Non-Deterministic (FOND) Model}
\begin{wrapfigure}{r}{0.42\textwidth}
    \vspace{-0.55cm}
    \begin{minipage}{0.42\textwidth}
    \begin{mdframed}[backgroundcolor=white,innerleftmargin=-0.8cm,linecolor=white]
    \begin{Verbatim}[commandchars=\\\{\}]
    {(}\PY{k}{:action}\PY{+w}{ }\PY{n+nx}{pick\PYZhy{}item}
    \PY{+w}{ }\PY{k}{:parameters}\PY{+w}{ }\PY{p}{(}\PY{n+nc}{?location}\PY{+w}{ }\PY{n+nc}{?item}\PY{p}{)}
    \PY{+w}{ }\PY{k}{:precondition}\PY{+w}{ }\PY{p}{(}\PY{n+nb}{and}
    \PY{+w}{   }\PY{p}{(}\PY{n+nx}{empty\PYZhy{}arm}\PY{p}{)}\PY{+w}{ }\PY{p}{(}\PY{n+nx}{has\PYZhy{}charge}\PY{p}{)}
    \PY{+w}{   }\PY{p}{(}\PY{n+nx}{robot\PYZhy{}at}\PY{+w}{ }\PY{n+nc}{?location}\PY{p}{)}
    \PY{+w}{   }\PY{p}{(}\PY{n+nx}{at}\PY{+w}{ }\PY{n+nc}{?location}\PY{+w}{ }\PY{n+nc}{?item}\PY{p}{)}\PY{p}{)}
    \PY{+w}{ }\PY{k}{:effect}\PY{+w}{ }\PY{p}{(}\PY{n+nb}{oneof}
    \PY{+w}{   }\PY{p}{(}\PY{n+nb}{and}\PY{+w}{ }\PY{p}{(}\PY{n+nb}{not}\PY{+w}{ }\PY{p}{(}\PY{n+nx}{empty\PYZhy{}arm}\PY{p}{)}\PY{p}{)}
    \PY{+w}{      }\PY{p}{(}\PY{n+nb}{not}\PY{+w}{ }\PY{p}{(}\PY{n+nx}{at}\PY{+w}{ }\PY{n+nc}{?location}\PY{+w}{ }\PY{n+nc}{?item}\PY{p}{)}\PY{p}{)}\PY{+w}{ }
    \PY{+w}{      }\PY{p}{(}\PY{n+nx}{holding}\PY{+w}{ }\PY{n+nc}{?item}\PY{p}{)}\PY{p}{)}
    \PY{+w}{   }\PY{p}{(}\PY{n+nb}{and}\PY{+w}{ }\PY{p}{(}\PY{n+nb}{not}\PY{+w}{ }\PY{p}{(}\PY{n+nx}{has\PYZhy{}charge}\PY{p}{)}\PY{p}{)}\PY{p}{)}\PY{p}{)}
    \end{Verbatim}
    \vspace{-0.2cm}
    \end{mdframed}
    \end{minipage}
    \caption{FOND description for the \emph{pick-item} capability of the cafe server robot.}
    \label{fig:fond}
    \vspace{-0.2cm}
\end{wrapfigure}    

A fully-observable
non-deterministic (FOND) planning model~\citep{Cimatti1998Strong} can be viewed as
a probabilistic planning model without the probabilities
associated with each effect pair. 
On executing an action, one of its possible effects 
is chosen. 
The solution to these planning models is a  
\emph{partial policy} $\Pi: S \rightarrow A$ that maps
each state to an action that the agent should execute in that state.
As shown by \citet{Cimatti1998Strong} and \citet{Daniele1999Strong}, the solution is a (i) \emph{weak solution} 
if the resulting plan may achieve the goal without any guarantee; 
(ii) \emph{strong solution} if the resulting plan is guaranteed to reach the goal; and 
(iii) \emph{strong cyclic solution} if the resulting plan is guaranteed to reach the goal under the assumption that 
in the limit, each action will lead to each of its effects.

Fig.~\ref{fig:fond} shows a sample FOND description of the pick-item capability (shown in Fig.~\ref{fig:ppddl}).
Note that there are no probabilities associated with each possible effect set, as the representation
only shows that one of these possible set of effects is possible on executing this capability. Also, the language
only supports the keyword \texttt{action}, hence it is used for representing capability in the first line.
    
\subsection{PPDDL}
We use Probabilistic Planning Domain Definition Language (PPDDL) to represent the 
probabilistic models in our work. It has two main components: (i) a domain description, consisting of definitions of the
actions that are possible along with their preconditions and effects (an example of action description is shown in Fig.~\ref{fig:ppddl}); and (ii) a problem description representing all the objects 
in the environment, 
a fully defined initial state, and partial description of the goal state. 

Sample domain and problem descriptions for the driver agent are available below:

\begin{Verbatim}[commandchars=\\\{\}]
\PY{p}{(}\PY{k}{define}\PY{+w}{ }\PY{p}{(}\PY{k}{domain}\PY{+w}{ }\PY{n+nx}{driver\PYZhy{}agent}\PY{p}{)}
\PY{+w}{  }\PY{p}{(}\PY{k}{:requirements}\PY{+w}{ }\PY{k}{:typing}\PY{+w}{ }\PY{k}{:strips}\PY{+w}{ }\PY{k}{:probabilistic\PYZhy{}effects}\PY{p}{)}
\PY{+w}{  }\PY{p}{(}\PY{k}{:types}\PY{+w}{ }\PY{n+nc}{location}\PY{p}{)}
\PY{+w}{  }\PY{p}{(}\PY{k}{:predicates}
\PY{+w}{       }\PY{p}{(}\PY{n+nx}{vehicle\PYZhy{}at}\PY{+w}{ }\PY{n+nc}{?l}\PY{+w}{ }\PY{n+nc}{\PYZhy{}}\PY{+w}{ }\PY{n+nx}{location}\PY{p}{)}
\PY{+w}{       }\PY{p}{(}\PY{n+nx}{spare\PYZhy{}in}\PY{+w}{ }\PY{n+nc}{?l}\PY{+w}{ }\PY{n+nc}{\PYZhy{}}\PY{+w}{ }\PY{n+nx}{location}\PY{p}{)}
\PY{+w}{       }\PY{p}{(}\PY{n+nx}{road}\PY{+w}{ }\PY{n+nc}{?from}\PY{+w}{ }\PY{n+nc}{\PYZhy{}}\PY{+w}{ }\PY{n+nx}{location}\PY{+w}{ }\PY{n+nc}{?to}\PY{+w}{ }\PY{n+nc}{\PYZhy{}}\PY{+w}{ }\PY{n+nx}{location}\PY{p}{)}
\PY{+w}{       }\PY{p}{(}\PY{n+nx}{not\PYZhy{}flattire}\PY{p}{)}\PY{p}{)}

\PY{+w}{  }\PY{p}{(}\PY{k}{:action}\PY{+w}{ }\PY{n+nx}{move\PYZhy{}vehicle}
\PY{+w}{    }\PY{k}{:parameters}\PY{+w}{ }\PY{p}{(}\PY{n+nc}{?from}\PY{+w}{ }\PY{n+nc}{\PYZhy{}}\PY{+w}{ }\PY{n+nx}{location}\PY{+w}{ }\PY{n+nc}{?to}\PY{+w}{ }\PY{n+nc}{\PYZhy{}}\PY{+w}{ }\PY{n+nx}{location}\PY{p}{)}
\PY{+w}{    }\PY{k}{:precondition}\PY{+w}{ }\PY{p}{(}\PY{n+nb}{and}\PY{+w}{ }\PY{p}{(}\PY{n+nx}{vehicle\PYZhy{}at}\PY{+w}{ }\PY{n+nc}{?from}\PY{p}{)}\PY{+w}{ }\PY{p}{(}\PY{n+nx}{road}\PY{+w}{ }\PY{n+nc}{?from}\PY{+w}{ }\PY{n+nc}{?to}\PY{p}{)}\PY{+w}{ }\PY{p}{(}\PY{n+nx}{not\PYZhy{}flattire}\PY{p}{)}\PY{p}{)}
\PY{+w}{    }\PY{k}{:effect}\PY{+w}{ }\PY{p}{(}\PY{n+nb}{and}\PY{+w}{ }\PY{p}{(}\PY{n+nx}{vehicle\PYZhy{}at}\PY{+w}{ }\PY{n+nc}{?to}\PY{p}{)}\PY{+w}{ }\PY{p}{(}\PY{n+nb}{not}\PY{p}{(}\PY{n+nx}{vehicle\PYZhy{}at}\PY{+w}{ }\PY{n+nc}{?from}\PY{p}{)}\PY{p}{)}
\PY{+w}{       }\PY{p}{(}\PY{n+nx}{probabilistic}\PY{+w}{ }\PY{l+m+mf}{0.8}\PY{+w}{ }\PY{p}{(}\PY{n+nb}{and}\PY{+w}{ }\PY{p}{(}\PY{n+nb}{not}\PY{p}{(}\PY{n+nx}{not\PYZhy{}flattire}\PY{p}{)}\PY{p}{)}\PY{p}{)}\PY{p}{)}\PY{p}{)}\PY{p}{)}
\PY{+w}{ }
\PY{+w}{ }\PY{p}{(}\PY{k}{:action}\PY{+w}{ }\PY{n+nx}{change\PYZhy{}tire}
\PY{+w}{    }\PY{k}{:parameters}\PY{+w}{ }\PY{p}{(}\PY{n+nc}{?l}\PY{+w}{ }\PY{n+nc}{\PYZhy{}}\PY{+w}{ }\PY{n+nx}{location}\PY{p}{)}
\PY{+w}{    }\PY{k}{:precondition}\PY{+w}{ }\PY{p}{(}\PY{n+nb}{and}\PY{+w}{ }\PY{p}{(}\PY{n+nx}{spare\PYZhy{}in}\PY{+w}{ }\PY{n+nc}{?l}\PY{p}{)}\PY{+w}{ }\PY{p}{(}\PY{n+nx}{vehicle\PYZhy{}at}\PY{+w}{ }\PY{n+nc}{?l}\PY{p}{)}\PY{+w}{  }\PY{p}{(}\PY{n+nb}{not}\PY{p}{(}\PY{n+nx}{not\PYZhy{}flattire}\PY{p}{)}\PY{p}{)}\PY{p}{)}
\PY{+w}{    }\PY{k}{:effect}\PY{+w}{ }\PY{p}{(}\PY{n+nb}{and}\PY{+w}{ }\PY{p}{(}\PY{n+nb}{not}\PY{+w}{ }\PY{p}{(}\PY{n+nx}{spare\PYZhy{}in}\PY{+w}{ }\PY{n+nc}{?l}\PY{p}{)}\PY{p}{)}\PY{+w}{ }\PY{p}{(}\PY{n+nx}{not\PYZhy{}flattire}\PY{p}{)}\PY{p}{)}\PY{p}{)}
\PY{p}{)}
\end{Verbatim}

\begin{Verbatim}[commandchars=\\\{\}]
\PY{p}{(}\PY{k}{define}\PY{+w}{ }\PY{p}{(}\PY{k}{problem}\PY{+w}{ }\PY{n+nx}{driver\PYZhy{}agent\PYZhy{}9}\PY{p}{)}
\PY{+w}{  }\PY{p}{(}\PY{k}{:domain}\PY{+w}{ }\PY{n+nx}{driver\PYZhy{}agent}\PY{p}{)}
\PY{+w}{  }\PY{p}{(}\PY{k}{:objects}\PY{+w}{ }\PY{n+nx}{l\PYZhy{}1\PYZhy{}1}\PY{+w}{ }\PY{n+nx}{l\PYZhy{}1\PYZhy{}2}\PY{+w}{ }\PY{n+nx}{l\PYZhy{}1\PYZhy{}3}\PY{+w}{ }\PY{n+nx}{l\PYZhy{}2\PYZhy{}1}\PY{+w}{ }\PY{n+nx}{l\PYZhy{}2\PYZhy{}2}\PY{+w}{  }\PY{n+nx}{l\PYZhy{}3\PYZhy{}1}\PY{+w}{ }\PY{n+nc}{\PYZhy{}}\PY{+w}{ }\PY{n+nx}{location}\PY{p}{)}

\PY{+w}{  }\PY{p}{(}\PY{k}{:init}
\PY{+w}{    }\PY{p}{(}\PY{n+nx}{vehicle\PYZhy{}at}\PY{+w}{ }\PY{n+nx}{l\PYZhy{}1\PYZhy{}1}\PY{p}{)}\PY{+w}{ }\PY{p}{(}\PY{n+nx}{not\PYZhy{}flattire}\PY{p}{)}\PY{+w}{ }\PY{p}{(}\PY{n+nx}{spare\PYZhy{}in}\PY{+w}{ }\PY{n+nx}{l\PYZhy{}2\PYZhy{}1}\PY{p}{)}\PY{+w}{ }\PY{p}{(}\PY{n+nx}{spare\PYZhy{}in}\PY{+w}{ }\PY{n+nx}{l\PYZhy{}2\PYZhy{}2}\PY{p}{)}\PY{+w}{ }
\PY{+w}{    }\PY{p}{(}\PY{n+nx}{spare\PYZhy{}in}\PY{+w}{ }\PY{n+nx}{l\PYZhy{}3\PYZhy{}1}\PY{p}{)}\PY{+w}{ }\PY{p}{(}\PY{n+nx}{road}\PY{+w}{ }\PY{n+nx}{l\PYZhy{}1\PYZhy{}1}\PY{+w}{ }\PY{n+nx}{l\PYZhy{}1\PYZhy{}2}\PY{p}{)}\PY{+w}{ }\PY{p}{(}\PY{n+nx}{road}\PY{+w}{ }\PY{n+nx}{l\PYZhy{}1\PYZhy{}2}\PY{+w}{ }\PY{n+nx}{l\PYZhy{}1\PYZhy{}3}\PY{p}{)}\PY{+w}{ }
\PY{+w}{    }\PY{p}{(}\PY{n+nx}{road}\PY{+w}{ }\PY{n+nx}{l\PYZhy{}1\PYZhy{}1}\PY{+w}{ }\PY{n+nx}{l\PYZhy{}2\PYZhy{}1}\PY{p}{)}\PY{+w}{ }\PY{p}{(}\PY{n+nx}{road}\PY{+w}{ }\PY{n+nx}{l\PYZhy{}1\PYZhy{}2}\PY{+w}{ }\PY{n+nx}{l\PYZhy{}2\PYZhy{}2}\PY{p}{)}\PY{+w}{ }\PY{p}{(}\PY{n+nx}{road}\PY{+w}{ }\PY{n+nx}{l\PYZhy{}2\PYZhy{}1}\PY{+w}{ }\PY{n+nx}{l\PYZhy{}1\PYZhy{}2}\PY{p}{)}\PY{+w}{ }
\PY{+w}{    }\PY{p}{(}\PY{n+nx}{road}\PY{+w}{ }\PY{n+nx}{l\PYZhy{}2\PYZhy{}2}\PY{+w}{ }\PY{n+nx}{l\PYZhy{}1\PYZhy{}3}\PY{p}{)}\PY{+w}{ }\PY{p}{(}\PY{n+nx}{road}\PY{+w}{ }\PY{n+nx}{l\PYZhy{}2\PYZhy{}1}\PY{+w}{ }\PY{n+nx}{l\PYZhy{}3\PYZhy{}1}\PY{p}{)}\PY{+w}{ }\PY{p}{(}\PY{n+nx}{road}\PY{+w}{ }\PY{n+nx}{l\PYZhy{}3\PYZhy{}1}\PY{+w}{ }\PY{n+nx}{l\PYZhy{}2\PYZhy{}2}\PY{p}{)}
\PY{+w}{  }\PY{p}{)}

\PY{+w}{  }\PY{p}{(}\PY{k}{:goal}\PY{+w}{ }\PY{p}{(}\PY{n+nb}{and}\PY{+w}{ }\PY{p}{(}\PY{n+nx}{vehicle\PYZhy{}at}\PY{+w}{ }\PY{n+nx}{l\PYZhy{}1\PYZhy{}3}\PY{p}{)}\PY{p}{)}\PY{p}{)}
\PY{p}{)}
\end{Verbatim}

\section{Additional Details}

\subsection{Instantiated Predicates}
A literal corresponding to a 
predicate $p \in \mc{P}$ can appear in
$\emph{pre}(c)$ or any $e_i(c)\in \emph{eff}(c)$ of a capability 
$c \in \mc{C}$ iff 
it can be instantiated using a subset of parameters of $c$. 
E.g., consider a capability \lf{move-vehicle (?src ?dest)} and a predicate
\lf{(connected ?x ?y)} in the example discussed earlier. Suppose a literal
corresponding to the predicate \lf{(connected ?x ?y)} can appear in the precondition 
and/or the effect of \lf{move-vehicle (?src ?dest)}. 
The possible lifted instantiations of 
predicate \lf{connected} compatible with \lf{move-vehicle} 
are \lf{(connected ?src ?dest)}, 
\lf{(connected ?dest ?src)}, 
\lf{(connected ?src ?src)}, and 
\lf{(connected ?dest ?dest)}.
The number of parameters in a predicate $p \in P$ that is relevant to a 
capability $c \in \mc{C}$, i.e., instantiated using a subset of parameters of $c$, is bounded by the maximum arity 
of $c$.
So using the capability names and the 
predicates, we get a set of instantiated predicates. In our implementation we use these set of instantiated predicates as the set of predicates.

\subsection{Abstraction Example} 
Fig.~\ref{fig:qps}(left) shows an example from the cafe server SDM setting, where a concrete low-level state is represented as xyz-coordinates, roll, pitch, and yaw values by the simulator. This state is then converted to a high-level state shown in the figure. We use the Boolean evaluation functions for evaluating each predicate. The state is represented as conjunction of the true predicates.

\begin{figure}[h]
    \centering
    \includegraphics[width=0.73\textwidth]{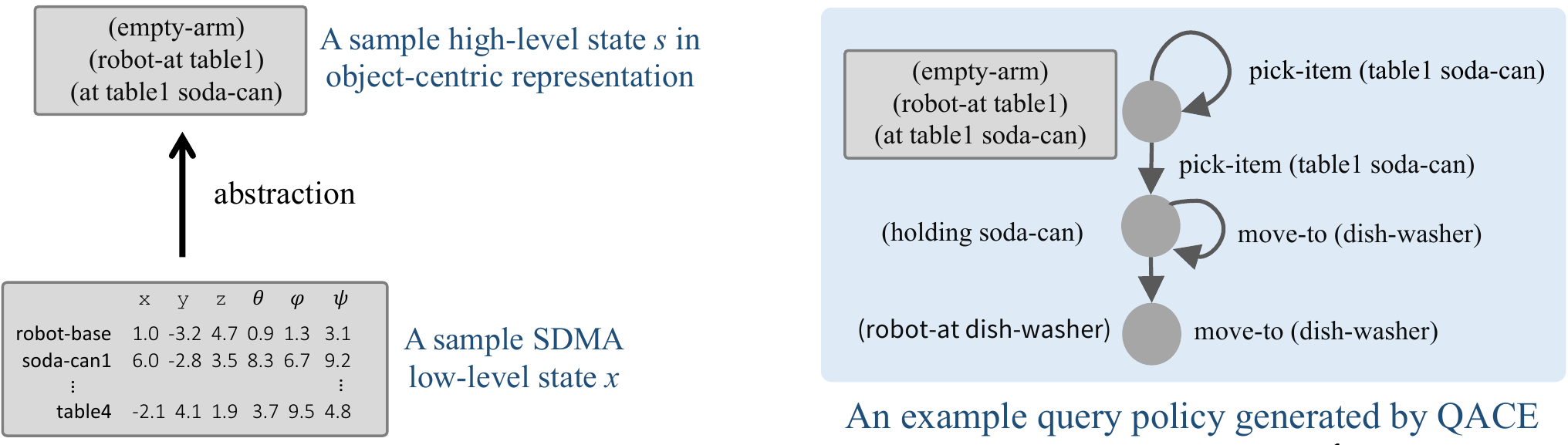}
    \caption{An example of abstraction of low-level state into a high level state (left) and an example of a policy simulation query (right). For the policy, the labels on the left of nodes correspond to state properties that must be true in those states, and the labels on right of edges correspond to the capabilities for each edge.
    The policy simulation query corresponds to: ``Given that the robot and \lf{soda-can} are at \lf{table1}, 
what will happen if the robot follows the following policy: if there is an item on the table and arm is empty, pick up the item; if an item is in the hand and location is not dishwasher, move to the dishwasher?''. }
    \label{fig:qps}
\end{figure}

\subsection{Example of Policy Simulation Query}
\label{app:psq_example}

Fig.~\ref{fig:qps}(right) shows an example of a policy simulation query. Note that the initial state is shown adjacent to the top-most node. 
A partial policy is a mapping from a partial state to a capability.
Such queries can be generated using non-deterministic planners like PRP~\citep{Muise2012PRP}. 

\mysssection{Generating Policy Simulation Queries using PRP}
QACE automates the generation of queries using FOND planning problems. QACE
always generates queries to distinguish between models that differ only on
one predicate corresponding to just one location (a precondition or effect in a capability). The main idea behind generating 
such queries is that the
responses to the query $q$ from the two models should be logically inconsistent.
To generate the policy simulation queries, QACE creates a FOND planning model and a problem.

Let $M_i$ and $M_j$ be a pair of FOND models expressed using $\mc{P}$ and $\mc{C}$ where $i,j \in \{T,F,I\}$.
QACE renames the predicates and capabilities in $M_i$ and $M_j$ as $\mc{P}_i$ and $\mc{P}_j$, and $\mc{C}_i$ and $\mc{C}_j$, respectively, so that there are no intersections and 
a pair of states in the two models can be progressed independently using pairs of capabilities.
This gives a planning model $M_{ij} = \tup{\mc{P}_{ij}, \mc{C}_{ij}}$.
Here, $\mc{P}_{ij} = \mc{P}_i \cup \mc{P}_j \cup \{\emph{(goal)}\}$, where $\emph{(goal)}$ is a 
0-ary predicate. It is used to identify when the goal for the FOND planning problem
is reached.
For each capability $\tup{c_i,c_j} \in \tup{\mc{C}_i,\mc{C}_j}$ such that their names match,
$\emph{pre}(c_{ij})$ of the combined capability $c_{ij}$ is disjunction of preconditions of $c_i$ and $c_j$. For 
$e(c_{ij}) \in \emph{eff}(c_{ij})$ ACE adds three conditional effects: (i) $\emph{pre}(c_i) \land \emph{pre}(c_j) \Rightarrow e(c_i) \land e(c_j)$;
(ii) $\emph{pre}(c_i) \land \neg\emph{pre}(c_j) \Rightarrow \emph{(goal)}$;  and
(iii) $\neg\emph{pre}(c_i) \land \emph{pre}(c_j) \Rightarrow \emph{(goal)}$.
An example of this process is included in the next section. 

Starting from an initial state, the FOND problem uses one
of these states and maintains two different
copies of all the objects in the environment, one corresponding to each of the
models. Each model only manipulates the objects in its own copy. QACE then 
solves a planning problem that has
an initial state $s_{I_{ij}} = \{p_i^{*_1},\dots,p_i^{*_z}, p_j^{*_1},\dots,p_j^{*_z} \}$ 
and a goal state $G_{ij} = (goal) \lor [\exists p \in \mc{P}_{ij}^* (p_i \land \neg p_j) \lor (\neg p_i \land p_j)]$. Here, $\mc{P}^*$ represents the
grounded version of predicates $\mc{P}$ using objects $O$ in the environment.
The partial policy $\pi$ generated as a solution to this 
planning problem is a \emph{strong solution}.
As shown by \citet{Cimatti1998Strong}, the solution is a 
\emph{strong solution} if the resulting plan is guaranteed to reach the goal.
The solution partial policy will lead the two models in a state where
at least one capability cannot be applied, and hence the \emph{(goal)} predicate becomes true.
This is possible because the models differ only in the way one predicate is added at a location. We formalize this with 
the following lemma. The proof is available in Appendix~\ref{appendix:proofs}.

\begin{lemma}
\label{ref:FOND_strong}
Given two models $M_i$ and $M_j$ such
that both are abstractions of the same FOND model,
and are at the same level of abstraction with only one predicate
differing in way it is added in one of the location, the intermediate FOND planning problem
created using QACE to generate policy simulation queries has a
strong solution.
\end{lemma}

We next an example of a sample planning domain and problem using which QACE generates a query. Recall that a FOND planning problem consists of two components, a planning domain and a planning problem. We will see an example of both below.

\mysssection{FOND planning domain} Consider we have a capability \lf{move-vehicle (?frm ?to)} in the Cafe server robot, and we already know one of its preconditions; \lf{(has-charge)}. We are now trying to find what will be the correct way to add the predicate \lf{(robot-at ?frm)} in the
precondition of this \lf{move-vehicle (?frm ?to)} capability. Consider we have two models $M_i$ and $M_j$, where $i=T$ and $j=F$. We will represent their
\lf{move-vehicle} capability as follows:

\begin{Verbatim}[commandchars=\\\{\}]
\PY{p}{(}\PY{k}{:action} \PY{n+nx}{move-vehicle\PYZus{}i}
  \PY{k}{:parameters} \PY{p}{(}\PY{n+nc}{?frm} \PY{n+nc}{\PYZhy{}} \PY{n+nc}{loc} \PY{n+nc}{?to} \PY{n+nc}{\PYZhy{}} \PY{n+nc}{loc}\PY{p}{)}
  \PY{k}{:precondition} \PY{p}{(}\PY{n+nb}{and} \PY{p}{(}\PY{n+nx}{has\PYZhy{}charge\PYZus{}i}\PY{p}{)}
    \PY{p}{(}\PY{n+nx}{robot\PYZhy{}at\PYZus{}i} \PY{n+nc}{?frm}\PY{p}{)}\PY{p}{)}
  \PY{k}{:effect} \PY{p}{(}\PY{n+nb}{and} \PY{p}{)}\PY{p}{)}
\end{Verbatim}

\begin{Verbatim}[commandchars=\\\{\}]
\PY{p}{(}\PY{k}{:action} \PY{n+nx}{move-vehicle\PYZus{}j}
  \PY{k}{:parameters} \PY{p}{(}\PY{n+nc}{?frm} \PY{n+nc}{\PYZhy{}} \PY{n+nc}{loc} \PY{n+nc}{?to} \PY{n+nc}{\PYZhy{}} \PY{n+nc}{loc}\PY{p}{)}
  \PY{k}{:precondition} \PY{p}{(}\PY{n+nb}{and} \PY{p}{(}\PY{n+nx}{has\PYZhy{}charge\PYZus{}j}\PY{p}{)}
    \PY{p}{(}\PY{n+nb}{not} \PY{p}{(}\PY{n+nx}{robot\PYZhy{}at\PYZus{}i} \PY{n+nc}{?frm}\PY{p}{)}\PY{p}{)}\PY{p}{)}
  \PY{k}{:effect} \PY{p}{(}\PY{n+nb}{and} \PY{p}{)}\PY{p}{)}
\end{Verbatim}

To create a query, we will combine the \lf{move-vehicle} capabilities into a combined capability. This should be done in such a way that the combined capability is executed when at least one of the model's preconditions are satisfied. Hence, for each capability $\tup{c_i,c_j} \in \tup{\mc{C}_i,\mc{C}_j}$ s.t. $\emph{name}(c_i) = \emph{name}(c_j)$,
$\emph{pre}(c_{ij}) = \emph{pre}(c_i) \lor \emph{pre}(c_j)$.

Now on executing this combined capability, we should be able to identify if the preconditions or effects of the capabilities in the two models $M_i$ and $M_j$ are different. To take these into account, the effect of the combined capability will be such that for each
$e(c_{ij}) \in \emph{eff}(c_{ij})$
we add three conditional effects: (i) $\emph{pre}(c_i) \land \emph{pre}(c_j) \Rightarrow e(c_i) \land e(c_j)$;
(ii) $\emph{pre}(c_i) \land \neg\emph{pre}(c_j) \Rightarrow \emph{(goal)}$;  and
(iii) $\neg\emph{pre}(c_i) \land \emph{pre}(c_j) \Rightarrow \emph{(goal)}$. 
The condition (ii) and (iii) helps identify that the precondition of the capability according to only one of the models $M_i$ or $M_j$ is satisfied. The condition (i) captures the case where the precondition of the capability according to both the models are satisfied. In this case, the effects of the capability according to both the models are applied.

Applying it here for the \lf{move-vehicle} 
capability, we get:

\begin{Verbatim}[commandchars=\\\{\}]
  \PY{p}{(}\PY{k}{:action} \PY{n+nx}{move-vehicle\PYZus{}ij}
    \PY{k}{:parameters} \PY{p}{(}\PY{n+nc}{?frm} \PY{n+nc}{\PYZhy{}} \PY{n+nc}{loc} \PY{n+nc}{?to} \PY{n+nc}{\PYZhy{}} \PY{n+nc}{loc}\PY{p}{)}
    \PY{k}{:precondition} \PY{p}{(}\PY{n+nb}{or}
      \PY{p}{(}\PY{n+nb}{and} \PY{p}{(}\PY{n+nx}{has\PYZhy{}charge\PYZus{}i}\PY{p}{)}
        \PY{p}{(}\PY{n+nx}{robot\PYZhy{}at\PYZus{}i} \PY{n+nc}{?frm}\PY{p}{)}\PY{p}{)}
      \PY{p}{(}\PY{n+nb}{and} \PY{p}{(}\PY{n+nx}{has\PYZhy{}charge\PYZus{}j}\PY{p}{)}
        \PY{p}{(}\PY{n+nb}{not} \PY{p}{(}\PY{n+nx}{robot\PYZhy{}at\PYZus{}j} \PY{n+nc}{?frm}\PY{p}{)}\PY{p}{)}\PY{p}{)}
    \PY{p}{)}
    \PY{k}{:effect} \PY{p}{(}\PY{n+nb}{and} 
      \PY{p}{(}\PY{n+nb}{when} \PY{p}{(}\PY{n+nb}{and} \PY{p}{(}\PY{n+nx}{has\PYZhy{}charge\PYZus{}i}\PY{p}{)}
          \PY{p}{(}\PY{n+nx}{robot\PYZhy{}at\PYZus{}i} \PY{n+nc}{?frm}\PY{p}{)}
          \PY{p}{(}\PY{n+nx}{has\PYZhy{}charge\PYZus{}j}\PY{p}{)}
          \PY{p}{(}\PY{n+nb}{not} \PY{p}{(}\PY{n+nx}{robot\PYZhy{}at\PYZus{}j} \PY{n+nc}{?frm}\PY{p}{)}\PY{p}{)}\PY{p}{)}
        \PY{p}{(}\PY{n+nb}{and} \PY{p}{)}
      \PY{p}{)}
      \PY{p}{(}\PY{n+nb}{when} \PY{p}{(}\PY{n+nb}{and} \PY{p}{(}\PY{n+nx}{has\PYZhy{}charge\PYZus{}i}\PY{p}{)}
          \PY{p}{(}\PY{n+nx}{robot\PYZhy{}at\PYZus{}i} \PY{n+nc}{?frm}\PY{p}{)}
          \PY{p}{(}\PY{n+nb}{or} \PY{p}{(}\PY{n+nb}{not} \PY{p}{(}\PY{n+nx}{has\PYZhy{}charge\PYZus{}j}\PY{p}{)}\PY{p}{)}
            \PY{p}{(}\PY{n+nx}{robot\PYZhy{}at\PYZus{}j} \PY{n+nc}{?frm}\PY{p}{)}\PY{p}{)}\PY{p}{)}
        \PY{p}{(}\PY{n+nb}{and} \PY{p}{(}\PY{n+nx}{goal}\PY{p}{)}\PY{p}{)}
      \PY{p}{)}
      \PY{p}{(}\PY{n+nb}{when} \PY{p}{(}\PY{n+nb}{and} \PY{p}{(}\PY{n+nx}{has\PYZhy{}charge\PYZus{}j}\PY{p}{)}
          \PY{p}{(}\PY{n+nb}{not} \PY{p}{(}\PY{n+nx}{robot\PYZhy{}at\PYZus{}j} \PY{n+nc}{?frm}\PY{p}{)}\PY{p}{)}
          \PY{p}{(}\PY{n+nb}{or} \PY{p}{(}\PY{n+nb}{not} \PY{p}{(}\PY{n+nx}{has\PYZhy{}charge\PYZus{}i}\PY{p}{)}\PY{p}{)}
            \PY{p}{(}\PY{n+nb}{not} \PY{p}{(}\PY{n+nx}{robot\PYZhy{}at\PYZus{}i} \PY{n+nc}{?frm}\PY{p}{)}\PY{p}{)}\PY{p}{)}\PY{p}{)}
        \PY{p}{(}\PY{n+nb}{and} \PY{p}{(}\PY{n+nx}{goal}\PY{p}{)}\PY{p}{)}
      \PY{p}{)}
    \PY{p}{)}
  \PY{p}{)}
\end{Verbatim}

Note that we have expanded
$\emph{pre}(c_i) \land \neg\emph{pre}(c_j)$ using disjunction of negations of all predicates in $\emph{pre}(c_j)$, etc. 
A pictorial example of a similar process for the \lf{has-charge} predicate in effects is shown in Fig.~\ref{fig:query_generation} below.

\begin{figure}[h]
    \centering
    \includegraphics[width=0.75\textwidth]{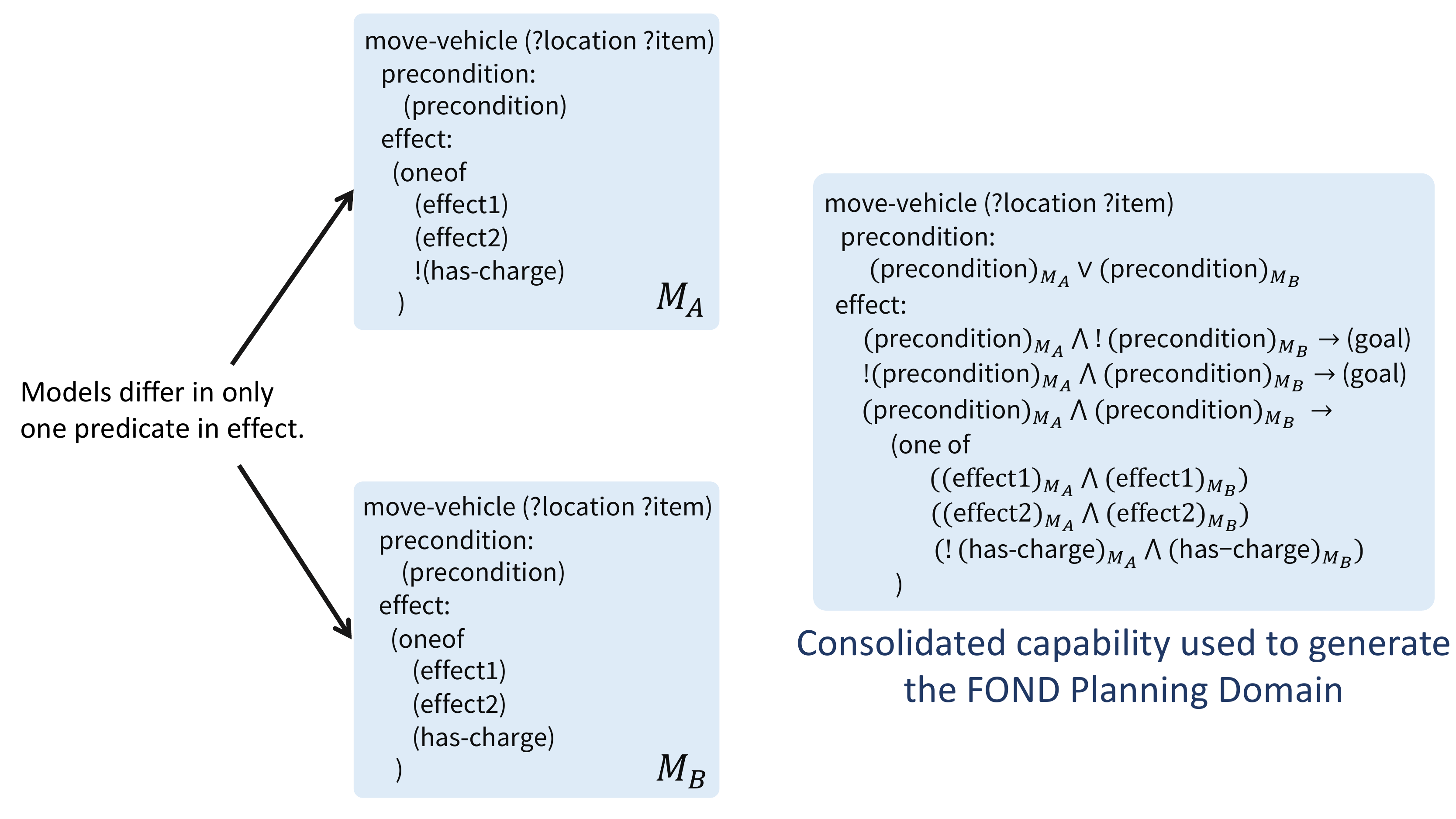}
    \caption{An example showing how two models $M_i$ and $M_j$ are combined to generate a FOND planning domain when the predicate is being added in effect of a capability. Note that the models only differ in one predicate having different form in both models.}
    \label{fig:query_generation}
\end{figure}

\mysssection{FOND planning problem}
The FOND planner must maintain two different
copies of all the objects in the environment, one corresponding to each of the
models $M_i$ and $M_j$. Each model only manipulates the objects in its own copy. 
The initial state of this planning problem $s_I$ is $\{p_i^{_1},\dots,p_i^{_z}, p_j^{_1},\dots,p_j^{_z} \}$. The goal formula $G$ is $(goal) \lor [\exists p \in \mathcal{P} (p_i \land \neg p_j) \lor (\neg p_i \land p_j)]$. Here $(goal)$ becomes true when a capability is executed by the policy such that the precondition of that capability is satisfied according to only one of the two models. The other condition captures the fact that a state is reached such that at least one of the predicate is true according to one of the models and false according to the other. This points to a difference in the effects of the capability that was executed last.

\subsection{Identifiable Effects}
\label{appendix:ident_effects}
A set of effects of a capability are identifiable if there exists a state such that 
when we execute a capability in that state, we can identify which of its effects was
executed. Let us consider a capability $a$, such that $\emph{pre}(a)= \{p_1 \land p_2 \land \neg p_3\}$, and
$\emph{eff}(a) = \{\tup{p_3 \land p_4 , 0.2}, \tup{p_3 \land \neg{p_2}, 0.5}, \tup{p_3 \land \neg p_4 \land \neg p_2, 0.3}\}$. The effects of this capability are identifiable because if we execute this capability in state $\{p1, p2, p4\}$, 
we can identify which of its effect is getting executed. This is because, on executing $a$, we can identify each effect as
follows: (i) if the resulting state has $p_4$ and $p_2$, then it is the first effect, (ii) if the resulting state has $p_4$
but not $p_2$, then it is the second effect, and (iii) if the resulting has neither $p_2$ nor $p_4$, then it is the third effect.

With all these concepts in place, we next see how one loop of the Alg.~\ref{alg:qace} is run.

\subsection{Example Run of the Algorithm}
Consider that the set of predicates consists of \lf{(has-charge)} and \lf{(robot-at ?frm)}, and \lf{move-vehicle} is one of the capability. Consider that we are starting with an empty model in line 2 of Alg.~\ref{alg:qace}. Now consider that the actual precondition of the \lf{move-vehicle} capability is \lf{(has-charge)} $\land$ \lf{(robot-at ?frm)}. The automated query generation process will involve executing the capability successfully in some state $s$ by the policy. The SDMA can only execute the capability in $s$ if \lf{(has-charge)} $\land$ \lf{(robot-at ?frm)} is true in $s$. As mentioned in Sec.~\ref{sec:learningmodels}, if $s$ doesn’t fulfill this criterion (i.e., the SDMA fails to execute the policy successfully) a new query is generated from a new initial state $s'$. Hence, this property of executing the capability in a state having \lf{(has-charge)} $\land$ \lf{(robot-at ?frm)} is ensured. Now, when reasoning about \lf{(has-charge)}, the policy can ask the agent to execute that capability in the state $s \setminus $\lf{(has-charge)}
 and if the SDMA fails to execute it then it means \lf{(has-charge)}
 is part of the precondition. Similarly, this can be done for \lf{(robot-at ?frm)} independently.

In the worst case, the search for a state $s$ where a query policy is executable will be exponential, but as the evaluations show, we can learn the correct model much faster. We also mention a way to overcome this in Sec.~\ref{sec:learningmodels}. Please note that even for methods like reinforcement learning, the worst-case upper bound is exponential in terms of the state space.

\mysssection{Possible models and their pruning} Now we see how QACE learns a correct model once it finds a state $s$ where a capability is executable by the agent. 
Consider that QACE is processing the tuple $\langle l,p \rangle$ = $\langle$precondition of \lf{move-vehicle}, \lf{(has-charge)}$\rangle$ in line 3 of Alg.~\ref{alg:qace}.

Now, QACE will generate the three models in line 4: (i) $M_T$ that has \lf{(has-charge)} as precondition of \lf{move-vehicle} capability; (ii) $M_F$ that has $\neg$\lf{(has-charge)} as precondition of \lf{move-vehicle} capability; and (iii) $M_I$ that has an empty precondition for \lf{move-vehicle} capability.

Consider that QACE is considering the pair $\langle M_T,M_F\rangle$ in line 5 of Alg.~\ref{alg:qace}. 
In the example being considered, executing the \lf{move-vehicle} capability in the state $s$  can help QACE distinguish between $M_T$ and $M_F$. Here the model $M_F$ will be unable to execute the \lf{move-vehicle} capability in $s$, whereas the model $M_T$ and the agent will be able to. So QACE will prune $M_F$ in line 8. 

Next, QACE will consider the pair $\langle M_T,M_I\rangle$ in line 5 of Alg.~\ref{alg:qace}. Here, to distinguish between these model, QACE will execute the \lf{move-vehicle} capability in a state $s'$ where \lf{(has-charge)} is false. Note that this state is also not generated manually, and the query generation does this autonomously, starting from the state $s$. Here the model $M_T$ and the agent will fail to execute the capability, whereas the model $M_I$ will succeed. Hence QACE can prune out $M_I$, leading it to learn the correct model $M^* = M_T$ where \lf{(has-charge)} is a precondition of the \lf{move-vehicle} capability.

Now starting with this updated current partial model $M^*$, consider that QACE picks  the tuple $\langle l,p \rangle$ = $\langle$precondition of \lf{move-vehicle}, \lf{(robot-at ?frm)}$\rangle$ in line 3 of Alg.~\ref{alg:qace}.
QACE will then generate three new models in next iteration in line 4: (i) $M_T$ that has \lf{(has-charge)} $\land$ \lf{(robot-at ?frm)} as precondition of \lf{move-vehicle} capability; (ii) $M_F$ that has $\neg$\lf{(has-charge)} $\land \neg$\lf{(robot-at ?frm)} as precondition of \lf{move-vehicle} capability; and (iii) $M_I$ that has \lf{(has-charge)} as precondition of \lf{move-vehicle} capability. So essentially, QACE builds upon the already learned partial model $M^*$ in previous iterations, and continues refining the model to eventually end up with the correct FOND model.

Once the correct set of preconditions and effects are learned, QACE counts the number of times each effect set was observed on executing each capability and perform the maximum likelihood estimation for each effect set to calculate the probabilities for each effect set. Note that a capability $c$ will at least appear in policies for all $\langle l, p \rangle$ pairs, such that location $l$ corresponds to a precondition or effect in $c$. So effectively, a capability $c$ can appear in at least $2 \times |\mc{P}|$ queries. So we will have at least $2 \times |\mc{P}| \times \eta$ samples for each capability.

\setcounter{lemma}{0}
\setcounter{theorem}{0}

\section{Theoretical Results}
\label{appendix:proofs}
This section provides proofs for the two theorems mentioned in the main paper.
We will first show that the plan in the distinguishing queries always ends up with the capability that is part
of the pal tuple being concretized at that time. This will help us in limiting our analysis to, at most,
the last 2 capabilities in the plan.

\begin{proposition}
\label{lem:last_action}
Let $M_i, M_j$, where $i,j \in \{T,F,I\}$, be the two models generated by
adding a predicate $p$ in a location corresponding to a capability $c$ to a model $M$. Suppose $q = \tup{s_I, \pi, G, \alpha, \eta}$ is a distinguishing query for two distinct models $M_i,
M_j$. The last capability in the partial policy $\pi$ to achieve $G$ will be $c$.
\end{proposition}
\begin{proof}
We prove this by contradiction. Consider that the last capability of the policy $\pi$ 
in the distinguishing
query $q$ is $c' \neq c$. Now the query $q$ used to distinguish between
$\m_i$ and $\m_j$ is generated using the FOND planning problem 
$\tup{M_{ij}, s_{I_{ij}}, G_{ij}}$, which has a
solution if both the models have different precondition or at least one different effect for
the same capability. Since the last capability of the policy is $c'$, the two models either
have different preconditions for $c'$ or different effects. This is not possible
as, according to Alg.~\ref{alg:qace}, $\m_i$ and $\m_j$ differ only in precondition or effect of
one capability $c$. Hence $c'=c$.
\end{proof}

We now use this proposition to prove Lemma~\ref{ref:FOND_strong} stated in Appendix~\ref{app:psq_example}.

\begin{lemma}
\label{ref:FOND_strong0}
Given two models $M_i$ and $M_j$ such
that both are abstractions of the same FOND model,
and are at the same level of abstraction with only one predicate
differing in way it is added in one of the location, the intermediate FOND planning problem
created using QACE to generate policy simulation queries has a
strong solution.
\end{lemma}

\begin{sproof}

We prove this in two parts. In the first part, we consider the case where we are refining 
the model in terms of the precondition of some capability. 
Recall that for each capability $c_{ij}$, we have 3 conditional effects:
i) $\emph{pre}(c_i) \land \emph{pre}(c_j) \Rightarrow e(c_i) \land e(c_j)$;
(ii) $\emph{pre}(c_i) \land \neg\emph{pre}(c_j) \Rightarrow \emph{(goal)}$;  and
(iii) $\neg\emph{pre}(c_i) \land \emph{pre}(c_j) \Rightarrow \emph{(goal)}$.
Now, according to proposition 1, capability $c_{ij}$ has to be the last capability in the policy $\pi$.
Since the model $M_i$ and $M_j$ differ only in preconditions, condition (ii) or (iii) must
be true for $c_{ij}$. This implies that on executing $c_{ij}$, the $(goal)$ predicate will become true, and 
executing this policy $\pi$ will end up in reaching the goal.

In the second part, we consider the case where we are refining 
the model in terms of the effects of some capability.
According to proposition 1, capability $c_{ij}$ has to be the last capability in the policy $\pi$.
Since the model $M_i$ and $M_j$ differ only in effects, condition (i) must be true for $c_{ij}$.
This implies that on executing $c_{ij}$, one of the predicates will become true according to one model,
and false according to another, and hence 
executing this policy $\pi$ will end up in reaching the goal condition $G_{ij}$.
\end{sproof}

Next, we prove the soundness and completeness of the learned model w.r.t. the agent model. Note that an important part of the process is to get a state $s$, where 
a capability $c$ can be executed successfully. We can collect this information
using some random traces, using a state where all capabilities are applicable, or
asking the agent for a state where certain conditions are met ($Q_{SR}$).
We use this information in the proof.

\begin{theorem}
\label{thm:thm1}
Let $\agent$ be a black-box SDMA with a ground truth transition model $\mc{T}'$ expressible in terms of predicates $\mc{P}$ and a set of capabilities $\mc{C}$. Let $M^*$ be the non-deterministic model expressed in terms of predicates $\mc{P}^*$ and capabilities $\mc{C}$, and learned using the query-based autonomous capability estimation algorithm (Alg.~\ref{alg:qace}) just before line 10. Let $\mc{C}_N$ be a set of capability names corresponding to capabilities $\mc{C}$. If $\mc{P}^* \subseteq \mc{P}$, then the model $M^*$ is \emph{sound} w.r.t. the SDMA transition model $\mc{T}'$. Additionally, if $\mc{P}^* = \mc{P}$, then the model $M^*$ is \emph{complete} w.r.t. the SDMA transition model $\mc{T}'$.
\end{theorem}

\begin{proof}
    We first prove that given the predicates $P$, capability names $\mc{C}_H$, model of the agent $\mc{T}'$, and the model 
    $M^*$ learned by Alg.~\ref{alg:qace}, $M^*$ is sound w.r.t. the model $\mc{T}'$. We do this in two cases. The first one showing
    that the learned preconditions of all the capabilities in $M^*$ are sound, and the second one showing the same thing for learned effects. 
    We use $M_T$, $M_F$, and $M_I$ to refer to models corresponding to adding $p$ (true), $\emph{not(p)}$ (false), and not adding $p$ (ignored), respectively to model $M^*$.\\

    \noindent
    Case 1: Consider the location is precondition in a capability $c$ where we are trying to find the correct way to add a predicate $p \in \mc{P}$.
    
    \noindent
    Case 1.1:
    Let the models we are comparing be $M_T$ and $M_I$ (or $M_F$).
    The policy simulation query $q$ to distinguish between these models would involve executing $c$ in a state where $p$ is false.
    Now, $M_T$ would fail to execute $c$ (as it has $p$ as a positive precondition), and $M_I$ (or $M_F$) would successfully execute it.
    If $\agent$ can execute $c$ in such a state, we can filter out the model $M_T$. We can also remove $p$ from a state where $\agent$
    is known to execute $c$, and see if it can execute $c$. If not, we can filter out the model $M_I$ (or $M_F$). \\
    \noindent
    Case 1.2:
    Let the models we are comparing be $M_F$ and $M_I$.
    The policy simulation query $q$ to distinguish between these models would involve executing $c$ in a state where $p$ is true.
    $M_F$ would fail to execute $c$ as it has $p$ as a negative precondition, whereas $M_I$ would successfully execute it.
    If $\agent$ can execute $c$ in such a state, we can filter out the model $M_T$. We can also add $p$ to a state where $\agent$
    is known to execute $c$, and see if it can execute $c$. If not, we can filter out the model $M_I$. \\

    \noindent
    Case 2: Consider the location is effect in a capability $c$ where we are trying to find the correct way to add a predicate $p \in \mc{P}^*$.
    
    \noindent
    Case 2.1:
    Let the models we are comparing be $M_T$ and $M_I$ (or $M_F$). The policy simulation query $q$ used to distinguish between these 
    models would involve executing $c$ in a state where $p$ is false. After executing it, the resulting state will have $p$
    true according to $M_T$ only. We ask the agent to simulate the policy $N$ times, with $p$ as the goal formula $G$.
    If $p$ appears in any of the simulation after executing $c$, then we learn all the possible effects involving $p$. Not that the capability has identifiable effects, so if $p$
    appears in more than one effect, the corresponding effect will eventually be discovered when concretizing the predicate that uniquely identifies that effect.\\
    Case 2.2:
    Let the models we are comparing be $M_F$ and $M_I$. The policy simulation query $q$ used to distinguish between these 
    models would involve executing $c$ in a state where $p$ is true. After executing it, the resulting state will have $p$
    true according to $M_I$ only. We ask the agent to simulate the policy $\eta$ times, with $p$ as the goal formula $G$.
    If $p$ appears in any of the runs, then we learn all the possible effects involving $p$. Not that the capability has identifiable effects, so if $p$
    appears in more than one effect, the corresponding effect will eventually be discovered when concretizing the predicate that uniquely identifies that effect.

    Combining both cases, we infer that whenever we learn a precondition or effect, it is added in the same form as in the ground truth model $\mc{T}'$,
    hence the learned model $M^*$ is sound w.r.t. $\mc{T}'$.

    We now prove that given the predicates $\mc{P}$, capability names $\mc{C}_H$, model of the agent $\mc{T}'$, and the model 
    $M^*$ learned by Alg.~\ref{alg:qace}, $M^*$ is complete w.r.t. the model $\mc{T}'$. 
    We just showed that the model that we learn is sound as whenever we add a predicate in a precondition
    or effect, it is in correct mode. Now, since Alg.~\ref{alg:qace} loops over all possible combinations of
    predicates and capabilities, for both precondition and effect, we will learn all the preconditions
    and effects correctly. Hence, the learned model will be complete w.r.t. the agent model.
\end{proof}

\begin{theorem}
  Let $\agent$ be a black-box SDMA with a ground truth transition model $\mc{T}'$ expressible in terms of predicates $\mc{P}$ and a set of capabilities $\mc{C}$. Let $M$ be the probabilistic model expressed in terms of predicates $\mc{P}^*$ and capabilities $\mc{C}$, and learned using the query-based autonomous capability estimation algorithm (Alg.~\ref{alg:qace}). Let $\mc{P} = \mc{P}^*$ and $M$ be generated using a sound and complete non-deterministic model $M^*$ in line 11 of Alg.~\ref{alg:qace}, and let all effects of each capability $c \in \mc{C}$ be identifiable. The model $M$ is \emph{correct} w.r.t. the model  $\mc{T}'$ in the limit as $\eta$ tends to $\infty$, where $\eta$ is hyperparameter in query $Q_\emph{PS}$ used in Alg.~\ref{alg:qace}.
\end{theorem}

\begin{sproof}
    Thm.~\ref{thm:thm1} showed that the model learned by Alg.~\ref{alg:qace} is sound and complete, meaning all the preconditions and effects are correctly learned.
    Consider that each sample generated by asking an agent to follow a policy is i.i.d.
    Now, if we consider only the samples in which a capability is applied in a state such that its effects are identifiable effects, then 
    we can use MLE to learn the correct probabilities given infinite such samples. This is a direct consequence of the result that given
    infinite i.i.d. samples, probabilities learned by maximum likelihood estimation converge to the true probabilities~\citep{Kiefer1956Consistency}.
\end{sproof}

\section{Extended Empirical Evaluation}
\label{appendix:results}

As mentioned earlier, we used a single, small training problem with few objects ($\le 7$).
To demonstrate generalizability, our test set contained problems that had twice the number of objects than the training problem.
Increasing the number of objects causes an exponential increase in the problem size in terms of the state space.

For all the experiments, for each run of the experiment, we run QACE as well as the baselines from scratch. For the plots, we took snapshot of the learned models every 60 seconds and computed the variational distance using a fixed test dataset.

In addition to the experiments described in the main paper, we also performed
some additional experiments. These are explained and discussed below.
\begin{figure*}[t]
\includegraphics[width=\textwidth]{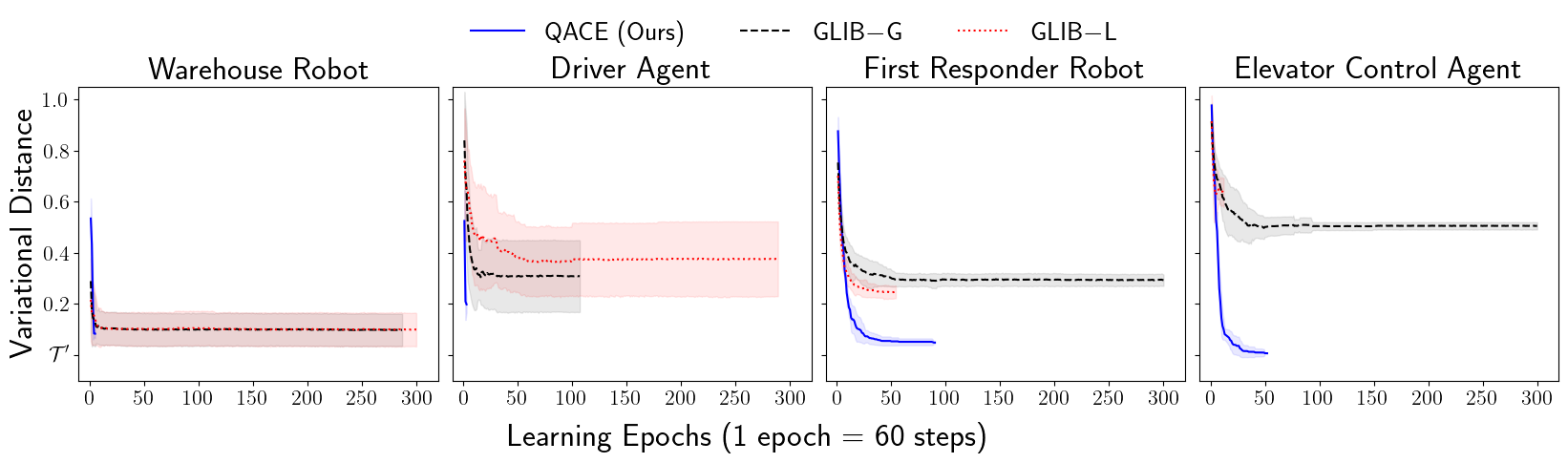}
\caption{Results showing the trends in the approximate Variational Distance w.r.t. the total number of steps in the environment (lower values better) for the three methods: QACE (ours), GLIB-G, and GLIB-L. Lines which do not extend until the end indicate that the time limit (4 hours) was exceeded.
The results were calculated using 30 runs per method per domain. Solid lines are averages across runs, and shaded portions show the standard deviation. $\mathcal{T}'$ is the ground truth model.}
\label{fig:step_plots}
\end{figure*}

\mysssection{Results w.r.t. environment steps} Fig.~\ref{fig:step_plots} show a comparison of the approximate variational distance
between QACE and the baselines as a factor of the total steps taken in the environment. From the results, it is clear that
QACE is able to outperform GLIB while taking far fewer steps in the environment. GLIB-L operates by babbling lifted goals and
we found that the goal babbling step of GLIB-L took an inordinate amount of time leading to very few steps in the environment
before the timeout of 4 hours. GLIB-G babbles grounded goals and thus can perform many steps but is not sample efficient in
learning as the results show. We analyzed the cause and found that if GLIB-G learns an incorrect model, it is often quite
difficult to get out of local minima since it keeps generating and following the same plan.

\begin{figure*}[t]
\includegraphics[width=\textwidth]{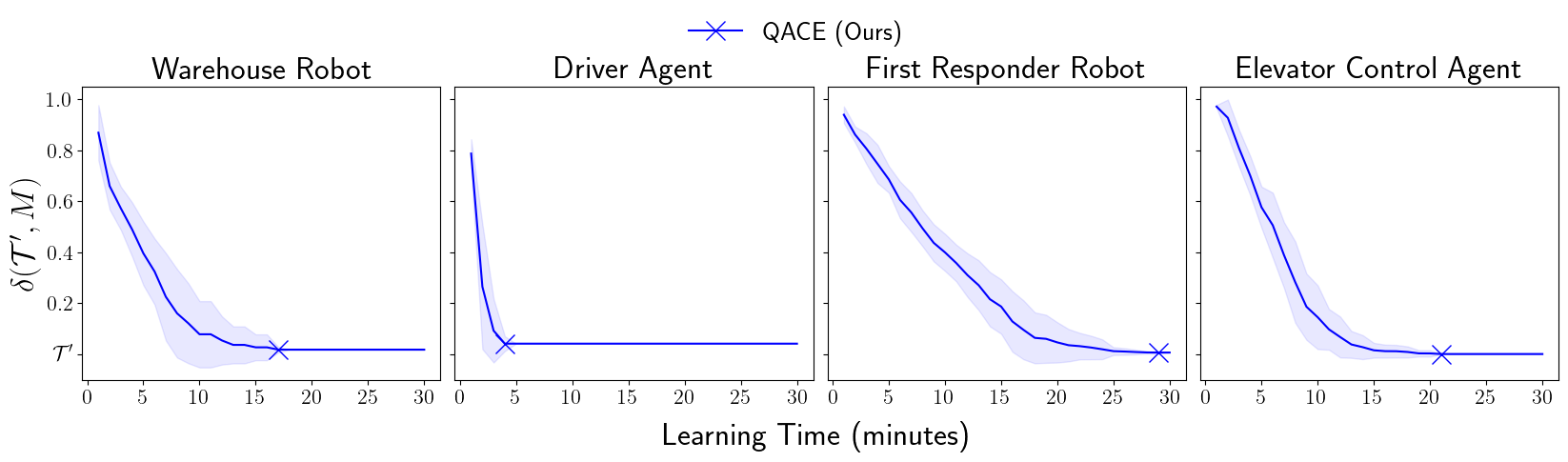}
\caption{Results showing the comparison of QACE w.r.t. the ground truth model $\mathcal{T}'$.
The plots show a trend in the variational distance (see Eq.~1) as a factor of the learning time for QACE (lower values better). {\color{blue} $\mathbf{\times}$} shows that the learning process ended
at that time instance for QACE.
The results were calculated using 30 runs per method per domain. Solid lines are averages across runs, and shaded portions show the standard deviation.}
\label{fig:gt_plots}
\end{figure*}

\mysssection{Evaluation w.r.t. ground truth models $\mathcal{T}'$} Fig.~\ref{fig:gt_plots} demonstrate that QACE is able to
converge to a learned model that is near-perfect compared to the ground truth model $\mathcal{T}'$.  QACE is able to learn
such a near-perfect model in a fraction of the time compared to the baselines (see Fig.~\ref{fig:plots} in the main paper). QACE can learn the
non-deterministic effects and preconditions in a finite number of representative environment interactions and given
enough samples MLE estimates are guaranteed to converge. This is in stark contrast to GLIB whose learned NDRs cannot
be easily compared to the ground truth.

\begin{figure}[ht]
    \centering
    \includegraphics[width=\textwidth]{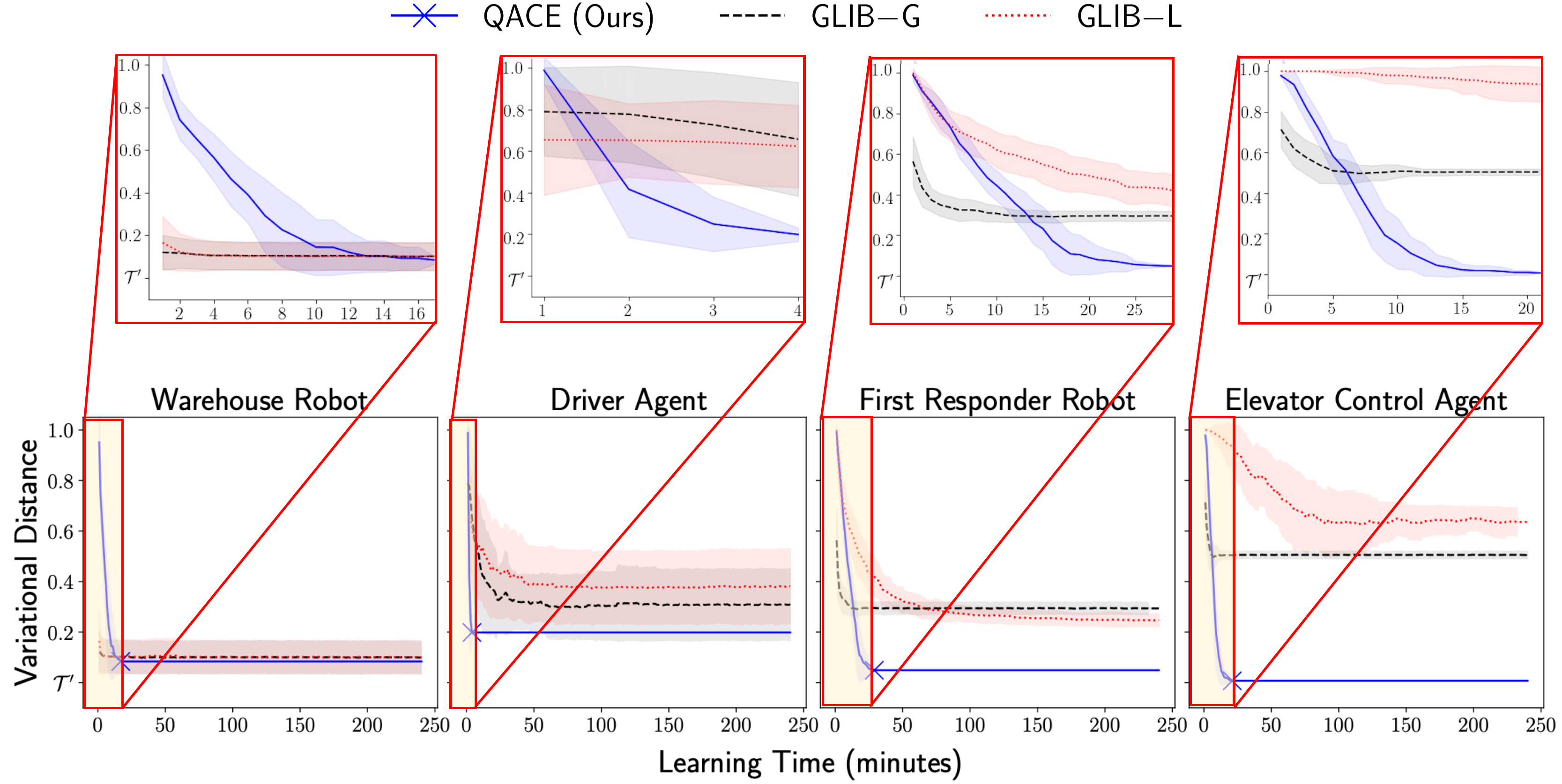}
    \caption{Results showing the comparison of QACE w.r.t. the ground truth model $\mathcal{T}'$.
The plots show a trend in the variational distance (see Eq.~1) as a factor of the learning time for QACE (lower values better). {\color{blue} $\mathbf{\times}$} shows that the learning process ended
at that time instance for QACE.
The results were calculated using 30 runs per method per domain. Solid lines are averages across runs, and shaded portions show the standard deviation. 
The zoomed in version shows the plots till learning process for QACE ends (marked using {\color{blue} $\mathbf{\times}$} in the zoomed-out plots). Note that QACE does not run beyond this.}
    \label{fig:zoomed-in}
\end{figure}

\mysssection{Faster convergence}
Fig.~\ref{fig:zoomed-in} shows a zoomed in version of the Fig.~\ref{fig:plots} in the main paper.  As you notice in the graph, the variational distance is very high initially, and it drops till the learning process of QACE ends (marked by {\color{blue} $\mathbf{\times}$} on the plots). We do not need to run QACE beyond this point and this time is short for all the domains. On the other hand, GLIB does not have a clear ending criterion. Hence we let it run for 4 hours and see that even with the extra time (and hence extra samples), it cannot learn a better model. The zoomed in plots also show that QACE does not learn the correct model in a one-shot manner, and that it actually keeps getting better with time as it processes more predicate and capability pairs.

\begin{figure}[t]
    \centering
    \includegraphics[width=\textwidth]{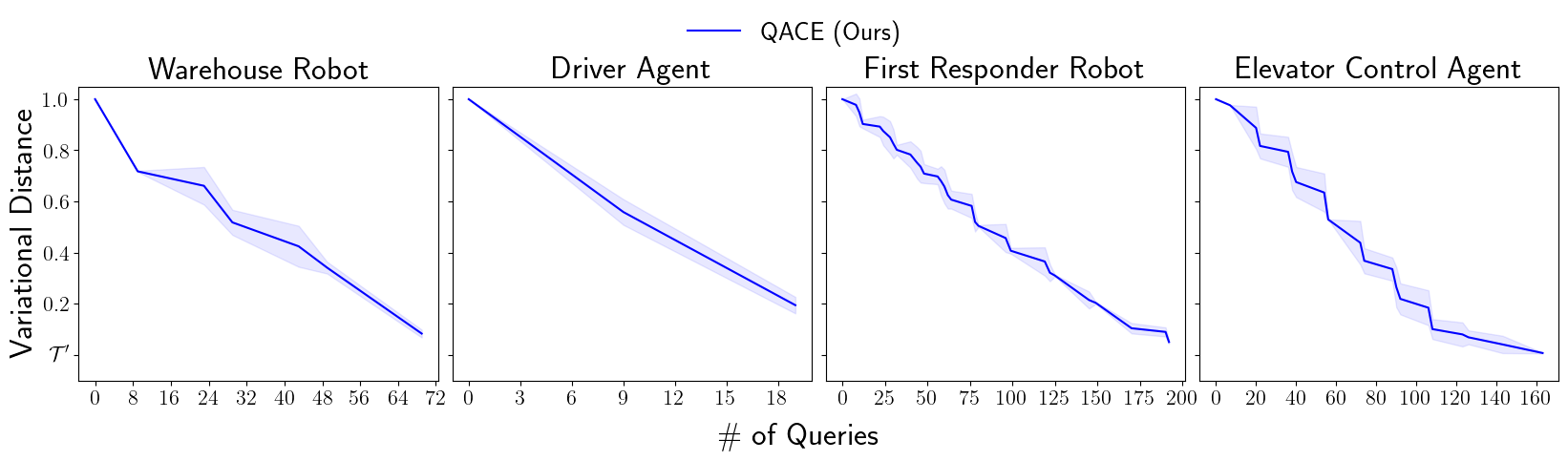}
    \caption{Results showing the avg. number of queries issued by QACE across 30 runs to achieve a specific variational distance (VD). Shaded regions represent one std. deviation. A VD of 0 (zero) corresponds to the ground truth model $\mathcal{T}'$.}
    \label{fig:query-count}
\end{figure}

\mysssection{Scalability} The number of queries needed to learn the model are linear in terms of the number of predicates and capabilities (\textit{for} loop in line 3 of Alg.~\ref{alg:qace}). The total number of queries for each domain shown in Fig.~\ref{fig:query-count} also correlates with the size of the domain shown in Tab.~\ref{tab:size}, supporting this hypothesis. Note that the \textit{for} loop in line 5 of Alg.~\ref{alg:qace} only contributes to a constant factor in the running time, as only three models are possible when adding a predicate at a location.

\end{document}